\documentclass[10pt]{article}

\usepackage[utf8]{inputenc}
\usepackage[T1]{fontenc}
\PassOptionsToPackage{hyphens}{url}\usepackage{url}
\usepackage{xurl}
\usepackage{booktabs}
\usepackage{graphicx}
\usepackage{amsmath,amssymb,amsfonts}
\usepackage{amsthm}
\usepackage{mathtools}
\usepackage{nicefrac}
\usepackage{microtype}
\usepackage{pifont}
\usepackage{bm}

\usepackage[ruled, lined, linesnumbered, commentsnumbered, longend]{algorithm2e}
\SetAlCapFnt{\small}\SetAlCapNameFnt{\small}
\usepackage{multirow}
\usepackage{physics}

\usepackage{array}
\usepackage{xr}
\usepackage{wrapfig}
\usepackage{setspace}
\usepackage{exscale}
\usepackage{relsize}
\usepackage{ragged2e}
\usepackage[table,xcdraw]{xcolor}
\usepackage{tikz}
\usepackage{echo_arxiv}
\usepackage{adjustbox}
\usetikzlibrary{arrows.meta,positioning,fit,backgrounds,calc,shapes.geometric}
\usetikzlibrary{decorations.pathreplacing}

\definecolor{pA}{RGB}{70,120,200}\definecolor{pB}{RGB}{225,140,60}
\definecolor{pC}{RGB}{90,170,100}\definecolor{pD}{RGB}{150,110,190}
\definecolor{segzero}{HTML}{426AB3}
\definecolor{segone}{HTML}{D77A32}
\definecolor{segtwo}{HTML}{438E72}
\definecolor{segthree}{HTML}{8661B3}
\definecolor{ink}{HTML}{243245}
\definecolor{muted}{HTML}{607084}
\definecolor{grid}{HTML}{D8E0E8}
\definecolor{soft}{HTML}{F3F6F9}
\definecolor{tileaccent}{HTML}{247F92}
\definecolor{tilefill}{HTML}{E6F1F3}
\newcommand{\figlab}{\footnotesize}
\newcommand{\fignote}{\scriptsize}
\newcommand{\FN}{45}
\tikzset{
  ecfig/.style={
    font=\figlab,
    every node/.append style={font=\figlab},
    >={Stealth[length=2.0mm]},
    line width=0.4pt},
  cell/.style={draw=black!45, line width=0.4pt, inner sep=0pt},
  hdr/.style={font=\figlab, text=black!80},
  note/.style={font=\fignote, text=black!60},
  err/.style={font=\figlab, text=carnelian},
  flow/.style={-{Stealth[length=2.0mm]}, line width=0.7pt, black!60},
  tap/.style={-{Stealth[length=2.0mm]}, line width=0.7pt, black!50},
}

\definecolor{oxfordblue}{rgb}{0.0, 0.13, 0.28}
\definecolor{carnelian}{rgb}{0.7, 0.11, 0.11}
\definecolor{darkslategray}{rgb}{0.18, 0.31, 0.31}
\definecolor{skyblue}{rgb}{0.53, 0.81, 0.92}
\definecolor{cadetblue}{rgb}{0.37, 0.62, 0.63}
\definecolor{oA}{RGB}{70,120,200}

\usepackage{hyperref}

\newcolumntype{P}[1]{>{\centering\arraybackslash}p{#1}}
\newtheorem{theorem}{Theorem}
\newtheorem{lemma}{Lemma}
\newtheorem{remark}{Remark}
\newtheorem{proposition}{Proposition}
\newtheorem{corollary}{Corollary}
\newtheorem{assumption}{Assumption}

\def\*#1{\mathbf{#1}}
\def\$#1{\mathcal{#1}}
\def\^#1{\mathbb{#1}}

\newcommand{\method}{\mbox{\textsc{Echo}}}

\newcommand{\TP}{P}
\newcommand{\Nsite}{N}
\newcommand{\bits}{b}
\newcommand{\gsz}{g}
\newcommand{\AllReduce}{\mathsf{AllReduce}}
\newcommand{\QAR}{\mathsf{QAR}}

\newcommand{\AllGather}{\mathsf{All\text{-}Gather}}

\newcommand{\injnoise}{\bm{\delta}}
\newcommand{\qbound}{q}

\newcommand{\site}{k}

\begin{document}

\title{\method{}: Error-Compensated 4-Bit Tensor-Parallel\\
Communication for LLM Training and Inference}

\author{%
  \mbox{Xingyu Xie$^{1,*}$} \hspace{1.6em}
  \mbox{Xiaoyong Xie$^{1,*}$} \hspace{1.6em}
  \mbox{Zhouchen Lin$^{2,\dagger}$} \hspace{1.6em}
  \mbox{Kim-Chuan Toh$^{1,\dagger}$}}
\affiliation{$^{1}$Department of Mathematics, National University of Singapore\\[2pt]
  $^{2}$State Key Laboratory of General Artificial Intelligence,
  School of Intelligence Science and Technology, Peking University}
\shorttitle{\method{}: Error-Compensated 4-Bit Tensor-Parallel Communication}
\preprintdate{October 2026}
\titlenote{$^{*}$Equal contribution. \quad $^{\dagger}$Corresponding authors: Zhouchen Lin (\href{mailto:zlin@pku.edu.cn}{zlin@pku.edu.cn}) and Kim-Chuan Toh (\href{mailto:mattohkc@nus.edu.sg}{mattohkc@nus.edu.sg}).}

\maketitle

\begin{abstract}
Tensor parallelism ends every layer with an all-reduce that sits between two
dependent matrix multiplications, so its cost cannot be hidden and grows as
compute outpaces interconnect bandwidth. Sending 4-bit instead of 16-bit values
would cut this cost by nearly 4$\times$, but the all-reduce must add the values it
carries, and rounding the sum injects an error at every call. Model quality
depends on how these errors accumulate over layers, yet existing methods treat
each call on its own. We present \method{}, a 4-bit all-reduce for training and
inference that controls the accumulated error. \method{} moves data in 4~bits
and sums in high precision, so each value is quantized at most twice per call
for any number of GPUs. The summation reveals each injected error, and
\method{} feeds it back into later calls, so old errors cancel instead of
accumulating. We prove that the accumulated error stays bounded independently
of depth. On A100 GPUs, \method{} is 2.09$\times$ faster than the 16-bit NCCL
all-reduce on 1\,GB messages. It speeds up the time to first token by up to
1.16$\times$ and the training step by up to 1.09$\times$ on an idle interconnect, and
by up to 1.44$\times$ and 1.38$\times$ as background traffic grows. It keeps the average accuracy over seven benchmarks
within 0.4 points of 16-bit inference on Qwen3-8B. In 10B-token pretraining,
its final loss is less than 0.9\% above that of 16-bit training.
\method{} needs no FP8 or FP4 hardware.

\vspace{4pt}\noindent\textbf{Keywords:} Tensor parallelism, collective communication, low-bit quantization, error
feedback, large language models.
\end{abstract}

\section{Introduction}\label{sec:intro}
Tensor parallelism splits the weight matrices of every
attention and MLP layer across several GPUs, and each GPU computes a partial
result from its share of the weights~\citep{megatronlm}. Turning the partial
results into the layer output requires an all-reduce, a collective operation
that adds the partial results from all GPUs and returns the sum to each of
them. The all-reduce runs between two dependent matrix multiplications, so it
is a synchronization barrier on the critical path. Unlike the gradient
all-reduce of data parallelism, which runs during the backward pass, or the
point-to-point transfers of pipeline parallelism, which are pipelined across
microbatches, it cannot be hidden behind the surrounding computation in the
same way. Overlapping can hide only part of it~\citep{domino}. In production
inference the all-reduce contributes up to 30\% of end-to-end
latency~\citep{metablog}. An 80-layer model issues 160 all-reduce calls in one
forward pass, and on GPUs with limited interconnect bandwidth this
communication can exceed half of the prefill cost~\citep{flashcomm}. In
large-scale training, tensor-parallel communication takes 50 to 60\% of total
communication time~\citep{taco}. Hardware trends make the problem worse every
generation, because peak compute has grown $3.0\times$ every two years while
interconnect bandwidth has grown only $1.4\times$~\citep{memorywall}. Kernel and decoding optimizations reduce the
compute cost of serving~\citep{flashattention,sfi} and leave the communication
untouched. Faster
collective algorithms and better use of spare links recover part of the
latency but leave the transferred bytes unchanged~\citep{flexlink,metablog}.
Reducing the cost further requires sending fewer bits.

Sending fewer bits means using a lower-precision format, and lower precision
has been adopted by every other part of the LLM stack. Matrix computation has
moved from FP16 to FP8 in production-scale training~\citep{deepseekv3} and to
FP4 on recent hardware~\citep{nvfp4}, with model quality preserved by careful
algorithm design. Model weights are served in 4-bit
formats~\citep{gptq,awq}. KV caches are compressed to 2 to 4
bits~\citep{kivi,kvquant}. Sharded data parallelism communicates weights and
gradients at close to 4 bits~\citep{sdp4bit}, and production mixture-of-experts systems
send expert-layer activations at FP8~\citep{deepseekv3,megascalemoe}.
Below 8 bits, tensor-parallel communication is
the exception. The all-reduce of most deployed systems still transfers 16-bit values, although
communication time is proportional to the bits sent and a 4-bit all-reduce
would nearly quadruple the effective interconnect bandwidth on unchanged hardware.
The reason is that the all-reduce must add the numbers it
moves, and a sum of 4-bit values does not fit in 4 bits.

The same constraint appears outside communication. Low-precision matrix
multiplication has the same problem, and tensor cores solve it
inside the chip. They multiply low-precision operands but accumulate the
products in high-precision registers, so the wide sum is never stored and
never moved~\citep{fp8formats,nvfp4}. The all-reduce cannot do the same. The
values it adds start on different GPUs and must cross the interconnect before
the addition, and the sum must cross the interconnect again to reach every
GPU. If both crossings are to stay at 4 bits, the collective must
dequantize what arrives, add in high precision, and quantize the sum again
for the way back. Every call injects a fresh rounding error into the hidden
state. The injections are also frequent. An LLM weight is quantized once, before
deployment, and only read afterwards, while the all-reduce quantizes at every
layer of every forward pass and, in training, of every backward pass over
tens of thousands of steps. The injected errors accumulate, and published
results show the consequences. In inference, quantizing tensor-parallel
communication to 4 bits without training costs measurable accuracy. On
Gemma 2 27B, ARC-challenge accuracy drops from 62.4 to 57.5 and WinoGrande
from 79.1 to 74.3, and keeping outlier channels at 16 bits recovers only part
of the loss~\citep{lowbittp}. In training, naive low-bit communication
destabilizes the loss curve or diverges~\citep{taco}. The most recent
training-side work identifies error accumulation under high-frequency
communication as the central obstacle, yet its remedy is a stronger quantizer
applied call by call, and it stops at FP8~\citep{taco}. Every existing method
works on each call in isolation, but low-bit communication is a dynamic process,
a sequence of rounding errors that flows through the network and through the
training run. Model quality depends on the accumulated error of this process,
which a per-call quantizer does not control.

We control the process rather than the individual calls. We
present \method{}, a tensor-parallel all-reduce that sends 4-bit integers,
for both inference and training. \method{} first
restructures the collective around the constraint that addition must happen
in high precision. A quantized all-to-all brings each data segment to one
owning GPU, the owner forms the sum in high precision, and a quantized
all-gather returns the reduced segments, so each element is quantized at most
twice per call (Sec.~\ref{sec:quant-allreduce}). On top of this primitive,
\method{} controls the error process. The reduction step reveals each rounding
error exactly, and a two-timescale compensator feeds it back into subsequent
calls, so the accumulated error stays bounded instead of growing with the
number of calls. The guarantee is independent of network depth and free of the
exponential Lipschitz factor of worst-case bounds (Secs.~\ref{sec:dual-ef} and~\ref{sec:theory}). The
name comes from this mechanism, in which each injected error returns to the
collective at the calls that follow, like an echo. The
compensation is fused into the entry pass of the all-reduce and into an update
that runs beside the next matrix multiplication, without changing
the computed values (Sec.~\ref{subsec:fusion}). The complete design transports 4-bit integers
with half-precision quantization parameters for each group and requires no FP8 or FP4 hardware support.

Our contributions are summarized as follows.
\begin{enumerate}
\item A quantized all-reduce primitive in which each element is quantized
  at most twice per call, independent of the number of GPUs. Computation
  and communication overlap inside the collective through a
  four-stage tiled pipeline that runs asynchronously over peer-to-peer direct
  writes, entirely on the GPUs with no CPU coordination between stages.
\item A two-timescale feedback compensator that measures the rounding error
  at the reduction step and cancels its accumulation across calls, keeping
  the total error bounded regardless of the number of calls, with the same
  mechanism serving inference and training.
\item A theory of the compensated collective. The bounds are independent of
  network depth and free of the exponential Lipschitz factor of worst-case
  bounds. The source bound holds for any
  error sequence on the forward pass and, at every site, in mean square on the backward pass
  under a smoothness condition. The propagation bound keeps the first-order
  output deviation, measured against the hidden state, at the level of a
  single rounding error at every depth.
\item A CUDA implementation that fuses compensation, quantization, and
  communication into one set of kernels. For 1\,GB messages, the compensated
  4-bit all-reduce is 2.09$\times$ faster than the 16-bit NCCL all-reduce, and the
  time to first token and the training step improve by up to 1.16$\times$ and
  1.09$\times$ on an idle interconnect and by up to 1.44$\times$ and
  1.38$\times$ as background traffic grows. 8-bit inference matches 16-bit accuracy, and 4-bit inference
  stays within 0.4 and 1.7 points of it on average over seven tasks for Qwen3-8B
  and Llama-3-8B. Over 10B tokens of pretraining, 4-bit training with a
  slightly larger learning rate ends less than 0.9\% above the 16-bit loss.
  To our knowledge, \method{} is the first to run tensor-parallel
  communication with 4-bit integers for both training and inference.
\end{enumerate}
The code will be released upon acceptance.

\section{Related Work}\label{sec:related}

We discuss three directions of related work: methods that hide or accelerate
the tensor-parallel all-reduce, methods that quantize it, and error-feedback
methods that compensate compression errors in other forms of parallelism. None
of them controls the accumulated error of a quantized all-reduce, which is the
central object of this paper.

The first direction hides or accelerates the all-reduce without reducing the
bits it moves. Domino and FLUX split the computation of a transformer layer
into pieces with no mutual dependence, so that one piece computes while
another communicates~\citep{domino,flux}. TokenWeave fuses the all-reduce with
the normalization that follows it, using multicast hardware that delivers one
write to all GPUs at once~\citep{tokenweave}. FlexLink routes part of the traffic over otherwise idle PCIe and RDMA
links~\citep{flexlink}. SHARP builds the addition into the network switches, so
the data is summed on its way through the network instead of on the
GPUs~\citep{sharp}. SwitchML, which aggregates gradients inside a programmable switch, must
first convert them to a fixed-point format with enough headroom to hold the
sum~\citep{switchml}, the same constraint our collective is built around
(Sec.~\ref{sec:quant-allreduce}). All of these methods transfer 16-bit
values. Hiding works only when enough independent computation exists to
run behind the communication, and this amount shrinks every hardware generation
because peak compute grows faster than interconnect
bandwidth~\citep{memorywall}. Our collective reduces the bytes instead, and a
scheduler can treat it as an ordinary all-reduce with less communication
to hide.

The second direction quantizes tensor-parallel communication directly and is
the closest to our work. For inference, Flash Communication (FlashComm) uses fine-grained
group quantization and a two-step all-reduce~\citep{flashcomm}.
FlashComm V2 supports arbitrary bit widths through bit splitting and
stores outliers in floating point~\citep{flashcommv2}. Mixed-precision schemes
keep outlier channels at 16 bits and reach an average of 4.2
bits~\citep{lowbittp}. Fine-grained quantization of the exchanged activations
shortens the time to first token by up to 2$\times$~\citep{hansenpalmus}, and
EQuARX builds a block-wise int8 all-reduce into the XLA compiler for
TPUs~\citep{equarx}. QuickReduce brings INT4, INT6, and INT8 all-reduce to
inference serving on AMD GPUs~\citep{quickreduce}. For training, TACO compresses the communicated
tensors to FP8 with a Hadamard transform and a two-level scaling
scheme~\citep{taco}, and MegaScale-MoE replaces the tensor-parallel
reduce-scatter of FP8 mixture-of-experts training with an FP8 all-to-all and an
FP32 sum~\citep{megascalemoe}. All of these techniques work within a single call. Finer
groups, preserved outliers, rotations, and wider formats reduce the error of
one quantization, but no method compensates the errors that earlier calls have
already injected, so the accumulated error over layers and steps remains
uncontrolled. \method{} controls
the accumulated error rather than the per-call error, and since its bounds scale
with the per-call error, these techniques can be combined with it. The compensator keeps
the accumulation bounded (Sec.~\ref{sec:dual-ef}), so a standard INT4
quantizer is sufficient for both inference and training, and each element is
quantized at most twice per call regardless of the number of GPUs
(Sec.~\ref{sec:quant-allreduce}).

The third direction compensates quantization errors in other forms of
parallelism. Error feedback is a classical technique from signal processing
and entered deep-learning training with 1-bit SGD, which stores the
quantization error of the gradient and adds it to the gradient of the next
step~\citep{onebitsgd}. Later analysis proved that this correction removes the
bias of aggressive compressors~\citep{efsignsgd}, and EF21 strengthened the
guarantees with a different mechanism that communicates a compressed
difference to a reference state~\citep{richtarik2021ef21}. PowerSGD, 1-bit
Adam, ZeRO++, and SDP4Bit bring compressed communication to large-scale
data-parallel training~\citep{powersgd,tang20211,wang2023zero++,sdp4bit}, and
LoCo shows that a moving-average estimate of the compression error, applied
before quantization, keeps low-bit gradient communication accurate at
scale~\citep{loco}. AQ-SGD extends compensation from gradients to the
activations at a pipeline-parallel boundary, quantizing the change of each example's activations since its previous
visit~\citep{aqsgd}. Error feedback works in all of these settings
because the compressed data is a gradient, or changes slowly over training.
The training loop produces a new gradient for the same parameters at every
step, so an error left uncorrected at one step can be added to the gradient
of the next step, and nothing is lost. A hidden state gives no such
opportunity. It is produced once, consumed by the next layer, and never
communicated again, and at inference time there is no training loop at all.
Accordingly, error feedback applied to the activations of model-parallel
training gives no gain over plain compression~\citep{rudakov2024}, and compression of
model-parallel activations has been studied empirically across pruning,
learned, and quantized schemes~\citep{bian2024}.
In tensor parallelism the data does not repeat, but the operation does. The
all-reduce runs again at the next layer, and in a pre-normalization
transformer all its outputs feed the same residual stream
(Sec.~\ref{sec:problem}). \method{} builds its feedback on this
repetition. The reduction step exposes every rounding error exactly
(Sec.~\ref{sec:dual-ef}), the compensator adds a correction at the
following calls, and the corrections cancel the accumulated error along the
residual stream, with a bound independent of depth
(Sec.~\ref{sec:theory}). To our knowledge, \method{} is the first to place
feedback compensation inside a quantized collective for tensor parallelism.

\section{Problem Formulation}\label{sec:problem}

\subsection{The quantized all-reduce and its error}\label{subsec:qar}

An all-reduce adds the contributions of the $\TP$ GPUs in the tensor-parallel
group, and this addition cannot be done in the low-bit transport format,
because a sum of low-bit values overflows its range. The sum is
formed in high precision, and the low-bit format only moves data between GPUs,
so a quantized all-reduce returns the exact sum plus the rounding error of the
data it moved. Section~\ref{sec:quant-allreduce}
states this constraint precisely and builds the decomposition that keeps the number
of roundings as small as possible.

Let $\mathbf{x}_{k,i}$ be the contribution held by GPU $i$ at the $k$-th all-reduce
of a pass, so an exact all-reduce returns $\sum_i \mathbf{x}_{k,i}$. The quantized
all-reduce, written $\QAR$, returns
\begin{equation}\label{eq:qar}
  \hat{\mathbf{y}}_k \;=\; \QAR\big(\{\mathbf{x}_{k,i}\}_{i=1}^{\TP}\big)
  \;=\; \sum_{i=1}^{\TP} \mathbf{x}_{k,i} \;+\; \injnoise_k .
\end{equation}
The vector $\injnoise_k$ is the error the quantized all-reduce injects at site $k$,
and it is the only difference between $\hat{\mathbf{y}}_k$ and the exact sum. Its
norm is bounded by a constant $\qbound$, $\norm{\injnoise_k}\le\qbound$, set by the
bit width $\bits$, the group size $\gsz$, and the scale of the data
(Sec.~\ref{sec:quant-allreduce}). Communication
volume is proportional to $\bits$, so a $\bits$-bit all-reduce sends about a fraction
$\bits/16$ of the bytes of a 16-bit all-reduce, the most that lower precision alone
can save.

\subsection{The communication noise process}\label{subsec:noise-process}

A pre-normalization transformer keeps a hidden state and, at each attention and
MLP block, normalizes the hidden state, runs the block, and adds the result
back~\citep{preln}. In tensor-parallel form each block produces its output in pieces,
so GPU $i$ computes a partial output $\mathbf{x}_{k,i}$ and an all-reduce sums the
$\TP$ pieces into the block output. Writing $\mathbf{z}_k$ for the hidden state
entering the $k$-th such block,
\begin{equation}\label{eq:residual}
  \mathbf{z}_{k+1} \;=\; \mathbf{z}_k \;+\; \sum_{i=1}^{\TP} \mathbf{x}_{k,i}.
\end{equation}
Each transformer layer has two of these all-reduces, one for its attention block
and one for its MLP block, so a pass has $N$ of them. Qwen3-0.6B, with 28 layers, has
$N=56$.

With the quantized all-reduce the block contributes
$\hat{\mathbf{y}}_k=\sum_i \mathbf{x}_{k,i}+\injnoise_k$ instead, so
\begin{equation}\label{eq:residual-quant}
  \mathbf{z}_{k+1} \;=\; \mathbf{z}_k \;+\; \sum_{i=1}^{\TP} \mathbf{x}_{k,i}
  \;+\; \injnoise_k .
\end{equation}
The useful block output $\sum_i \mathbf{x}_{k,i}$ is unchanged and the quantized
all-reduce only adds the error $\injnoise_k$. Low-bit tensor-parallel communication
is therefore a noise process on the
computation graph, injecting an error at each all-reduce site of a pass
(Fig.~\ref{fig:noise-process}). The same process occurs every time the model
runs. A single forward pass in inference injects $N$ such errors. In training the
same injection happens on the forward and the backward pass of every step, and over
a full run it repeats across many steps, so the process acts over far more sites.
The point of injection also differs between the two passes. On the forward pass
$\injnoise_k$ joins the hidden state directly, while on the backward pass it
enters the gradient of the hidden state after the normalization, so the forward
and backward errors are corrected separately (Sec.~\ref{subsec:backward-impl}).

\begin{figure}[t]
\begin{minipage}[c]{0.600\linewidth}\centering
\begin{adjustbox}{width=\linewidth}
\begin{tikzpicture}[ecfig,rounded corners=2pt,scale=0.97]
\def\dx{2.45}
\def\tapo{0.95}   
\def\po{0.95}   
\draw[->,line width=1.0pt,pA!75!black,rounded corners=0pt] (-1.35,0) -- (5.635+\po+0.45,0);
\node[hdr,anchor=south west] at (-1.33,0.07) {hidden state $\*{z}$};
\foreach \X/\kk in {0/0,2.45/1,5.635/{N\!-\!1}}{
  \node[cell,rounded corners=1.5pt,draw=pA!65!black,line width=0.6pt,fill=pA!\FN!white,
        minimum width=1.6cm,minimum height=0.70cm,align=center,font=\figlab]
       (B) at (\X,-1.30) {block $\kk$};
  \fill[pA!75!black] (\X-\tapo,0) circle (0.045);
  \draw[tap] (\X-\tapo,0) -- (\X-\tapo,-1.30) -- (B.west);
  \draw[->,line width=0.8pt,pA!75!black] (B.east) -- (\X+\po,-1.30)
        -- node[err,pos=0.62,left=3pt,fill=white,inner sep=1pt] {$+\,\injnoise_{\kk}$} (\X+\po,-0.22);
  \node[circle,draw=pA!75!black,line width=0.7pt,fill=white,inner sep=0pt,minimum size=3.4mm]
       at (\X+\po,0) {\scriptsize$+$};
}
\node[font=\figlab] at (4.04,-1.30) {$\cdots$};
\node[note,anchor=north] at (2.30,-1.92)
  {one all-reduce per block, so a forward pass has $N$ injection sites};
\end{tikzpicture}
\end{adjustbox}
\end{minipage}\hfill
\begin{minipage}[c]{0.370\linewidth}
\caption{\textbf{Low-bit tensor-parallel communication as a noise process.} A
pre-normalization transformer adds each attention and MLP block's output to the
hidden state. In tensor-parallel form the block ends in an all-reduce, and
replacing it with the quantized all-reduce leaves the block output in place while
adding an error $\injnoise_\site$ at that site. A forward pass has $N$ such
sites, and the backward pass injects at the same sites
(Sec.~\ref{subsec:backward-impl}).}\label{fig:noise-process}
\end{minipage}
\end{figure}
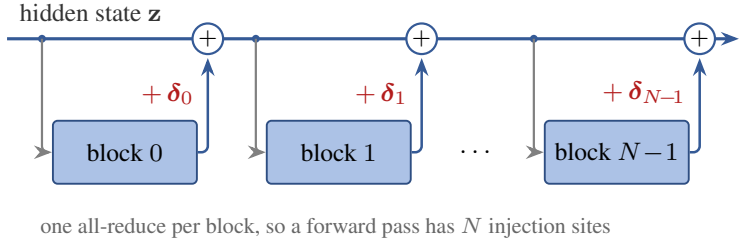

\subsection{Controlling the accumulated error}\label{subsec:controlling}

The output error is not set by any single $\injnoise_k$. Once injected, an error is
carried to the output by the blocks that follow it, which respond to the
already-perturbed hidden state and amplify it. The output error is built from
the $N$ injections after this amplification.
Its size is bounded by two factors. The first is how much the network amplifies a perturbation,
which Sec.~\ref{sec:theory} bounds under one assumption on the trained network. The second is how much the injections
reinforce one another as they build up, measured by the largest partial sum of the
injected sequence. The amplification is a property of the network, so the factor a
communication primitive can change is the partial sum.

The partial sum can be kept from growing because the quantized all-reduce reveals
its own error. The residual left on GPU $i$ after the all-reduce,
$\mathbf{u}_{k,i}=\mathbf{x}'_{k,i}-\hat{\mathbf{y}}_k/\TP$, with $\mathbf{x}'_{k,i}$ the (compensated) contribution it sends into the all-reduce (Sec.~\ref{sec:dual-ef}), sums across the GPUs to
\begin{equation}\label{eq:residual-identity}
  \sum_{i=1}^{\TP} \mathbf{u}_{k,i} \;=\; -\,\injnoise_k ,
\end{equation}
so the group recovers the error it has just injected, at no extra communication. The
primitive feeds this measurement back: before the next quantization it adds a
correction $\mathbf{C}_k$, built from the errors measured at earlier sites, to the
all-reduce input. Through the skip connection the correction enters the hidden state
beside the injection, so the net error added at the site is
\begin{equation}\label{eq:net-error}
  \bm{\rho}_k \;=\; \injnoise_k \;+\; \mathbf{C}_k .
\end{equation}
Because $\mathbf{C}_k$ is built from the earlier injected errors, it cancels their
buildup in the running sum, so the partial sums stay bounded as the pass proceeds.
This is error feedback, and Sec.~\ref{sec:dual-ef} gives the specific filter.

The factor to control is then the largest partial sum of these net errors. Writing
\begin{equation}\label{eq:partial-sum}
  \mathbf{S}_t \;=\; \sum_{k \le t} \bm{\rho}_k ,
\end{equation}
the requirement on the primitive is that $\max_t \norm{\mathbf{S}_t}$ does not grow with
the number of all-reduces. Without correction the partial sum is
$\sum_{k \le t} \injnoise_k$, whose largest value can reach $N\qbound$ over a
forward pass. This growth causes the accuracy loss of a plain 4-bit all-reduce
in inference and its instability in training. Since the amplification is
shared, inference and training both need a partial-sum bound that stays flat as
the number of sites grows. Section~\ref{sec:dual-ef} builds the
correction $\mathbf{C}_k$. Section~\ref{sec:theory} proves that the partial sum
$\max_t \norm{\mathbf{S}_t}$ stays bounded independently of $N$ for any error sequence on the forward pass, that its backward counterpart stays bounded in mean square at every site, and
that its first-order effect on the output, measured against the hidden state, stays at the level of a single rounding error. Section~\ref{sec:experiments} evaluates the
accuracy.

\section{Quantized P2P All-Reduce}\label{sec:quant-allreduce}

\method{} replaces the tensor-parallel (TP) all-reduce with a communication primitive
that moves data between GPUs in low precision but forms the sum in high precision.
Two properties of this design are used throughout the paper. First, the reduction
is carried out in high precision and introduces no error of its own, so the
all-reduce output differs from the exact sum only because the transported data is
quantized. Second, every value passes through the quantizer at most twice,
regardless of how many GPUs participate, so the number of roundings at each site
does not grow with the tensor-parallel degree $\TP$.
The rest of the section describes an implementation that realizes both
properties at low latency.

\subsection{A forced decomposition}\label{subsec:cannot-sum}

Turning the lower precision of Sec.~\ref{sec:problem} into a faster all-reduce
runs into a constraint. The all-reduce must add the contributions of the $\TP$
GPUs, and this addition cannot be carried out in the low-bit transport format. Two
4-bit values can already add up to a number that no longer fits in four bits, so a
low-bit accumulator overflows. Even where 4-bit arithmetic is available, four bits
are too few to hold a running partial sum without large rounding error, and low-bit
accumulation is in any case not widely supported across accelerators. The sum is
therefore formed in high precision, and the low-bit format only moves data between
GPUs. Carried out naively, a quantized all-reduce becomes a repeated cycle of
dequantize, add, and requantize, and each requantization injects one of the errors
$\injnoise_{\site}$ of Sec.~\ref{sec:problem}. The rest of this section keeps the
number of these requantizations as small as possible.

We first recall the standard structure of an all-reduce. Each GPU, one rank of the
TP group, holds a tensor split into $\TP$ equal segments, one associated with each
GPU. A \emph{reduce-scatter} sums the segments across GPUs so that GPU $j$ is left
with the sum of the $j$-th segments. An \emph{all-gather} then copies every summed
segment to every GPU, after which all GPUs hold the complete result. This
two-phase form is the bandwidth-optimal way to build an
all-reduce~\citep{patarasuk2009bandwidth}. This section treats a single all-reduce and drops the site index $\site$ of
Sec.~\ref{sec:problem}, writing $\*x_i$ for GPU $i$'s contribution $\*x_{\site,i}$
and $\*x_i^{\,j}$ for its $j$-th segment,
\begin{equation}\label{eq:decomp}
\begin{aligned}
  \*s^{\,j} &= \sum_{i=1}^{\TP} \*x_i^{\,j},\\[2pt]
  \AllReduce\big(\{\*x_i\}\big) &= \AllGather\big(\*s^{\,1},\dots,\*s^{\,\TP}\big).
\end{aligned}
\end{equation}
Because the segment sums $\*s^{\,j}$ must be formed in high precision, we implement
the reduce-scatter with an \emph{all-to-all}, the collective in which every GPU
sends a different piece to every other GPU. GPU $i$ sends its $j$-th segment to the
owner GPU $j$, which then holds the $\TP$ copies it needs and adds them locally. For
each segment it does not own, GPU $i$ quantizes that segment to $\bits$ bits and
writes it directly into GPU $j$'s buffer over the peer-to-peer link. The segment a
GPU owns is never sent over the link, so it stays in full precision and enters the
sum without any quantization error, and only the $\TP-1$ copies arriving from peers
are quantized. GPU $j$ dequantizes those $\TP-1$ copies, adds them to its own
full-precision segment in high precision to obtain $\*s^{\,j}$, and quantizes
$\*s^{\,j}$ once more for transport. The all-gather sends each quantized
$\*s^{\,j}$ back to the other GPUs, and every GPU dequantizes what it receives. GPU
$j$ reconstructs its own copy of $\*s^{\,j}$ from that same $\bits$-bit code rather
than keeping the high-precision value, so that its result is bit-for-bit identical
to the copies the other GPUs receive. Figure~\ref{fig:decomposition}(a) illustrates
the process for $\TP=4$, and Fig.~\ref{fig:decomp-detailed} in
Appendix~\ref{app:abi} works the same case through with numbers.

Among bandwidth-optimal schemes, this decomposition also uses the fewest
quantization steps. A ring all-reduce
passes partial sums from GPU to GPU and requantizes each value once per hop,
$\order{\TP}$ times in total. In our decomposition a value is quantized only
when it crosses the link, once in the all-to-all and once in the all-gather,
so each value is rounded at most twice however many GPUs participate
(Table~\ref{tab:vs-ring}). The two-phase skeleton
itself is standard and is shared by quantized systems such as
ZeRO++~\citep{wang2023zero++}, MegaScale-MoE~\citep{megascalemoe}, and
FlashComm~\citep{flashcomm,flashcommv2}. We
adopt it because it minimizes the roundings. On
top of this, the compensator keeps the error accumulated over a whole forward
pass at about twice the error of a single site, however deep the network
(Secs.~\ref{sec:dual-ef} and~\ref{sec:theory}).

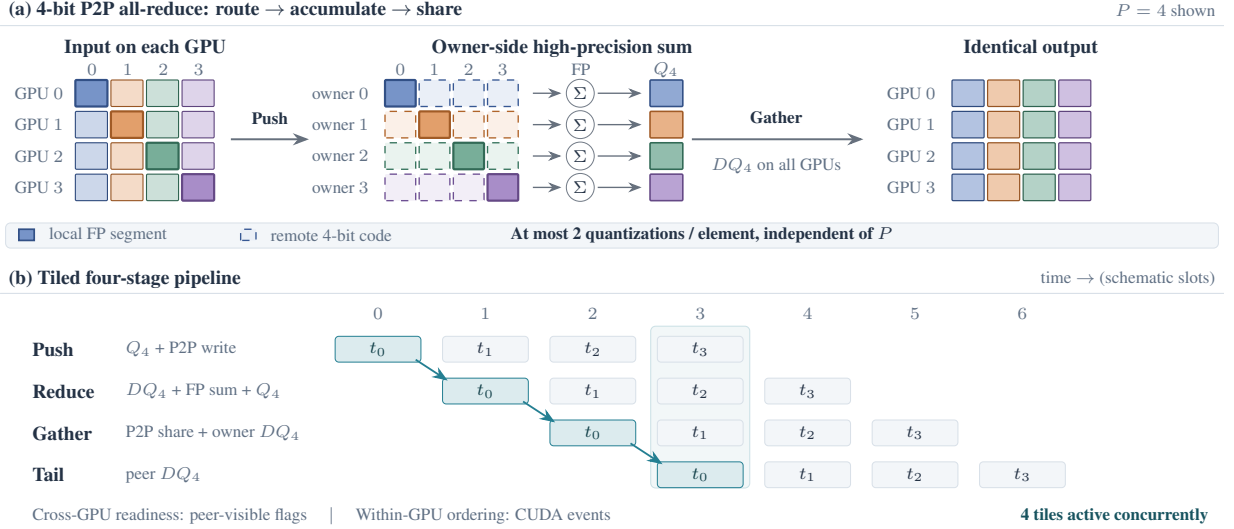
\begin{figure*}[t]
\centering
\resizebox{\textwidth}{!}{%
\begin{tikzpicture}[
  x=1cm,y=1cm,>=Stealth,
  every node/.style={font=\footnotesize,text=ink},
  tinytext/.style={font=\scriptsize,text=muted},
  heading/.style={font=\footnotesize\bfseries,text=ink},
  flow/.style={->,line width=0.8pt,draw=ink!70},
  op/.style={rounded corners=1.5pt,draw=grid,fill=soft,inner sep=2pt},
  slot/.style={rounded corners=1.5pt,draw=grid,fill=soft,
               minimum width=1.23cm,minimum height=0.37cm,inner sep=0pt,
               font=\scriptsize},
  firstslot/.style={slot,draw=tileaccent,fill=tileaccent!16,
                    text=tileaccent!50!black,font=\scriptsize\bfseries}
]
\def\segcol#1{\ifcase#1 segzero\or segone\or segtwo\or segthree\fi}

\node[heading,anchor=west] at (0,0.06) {(a) 4-bit P2P all-reduce: route $\rightarrow$ accumulate $\rightarrow$ share};
\node[tinytext,anchor=east] at (17.55,0.06) {$P=4$ shown};
\draw[grid,line width=0.45pt] (0,-0.16)--(17.55,-0.16);

\node[heading] at (2.08,-0.48) {Input on each GPU};
\node[heading] at (8.06,-0.48) {Owner-side high-precision sum};
\node[heading] at (14.78,-0.48) {Identical output};

\foreach \j/\c in {0/segzero,1/segone,2/segtwo,3/segthree}{
  \node[tinytext] at ({1.32+0.51*\j},-0.79) {$\j$};
}
\foreach \i in {0,1,2,3}{
  \node[tinytext,anchor=east] at (1.03,{-1.13-0.45*\i}) {GPU $\i$};
  \foreach \j/\c in {0/segzero,1/segone,2/segtwo,3/segthree}{
    \ifnum\i=\j
      \draw[draw=\c!75!black,fill=\c!72,line width=0.9pt,rounded corners=0.5pt]
        ({1.08+0.51*\j},{-1.32-0.45*\i}) rectangle ++(0.46,0.38);
    \else
      \draw[draw=\c!70!black,fill=\c!27,line width=0.45pt,rounded corners=0.5pt]
        ({1.08+0.51*\j},{-1.32-0.45*\i}) rectangle ++(0.46,0.38);
    \fi
  }
}

\draw[flow] (3.31,-1.79)--(4.43,-1.79);
\node[font=\scriptsize\bfseries,fill=white,inner sep=1pt] at (3.87,-1.49) {Push};

\foreach \j/\c in {0/segzero,1/segone,2/segtwo,3/segthree}{
  \node[tinytext,anchor=east] at (5.41,{-1.13-0.45*\j}) {owner $\j$};
  \foreach \i in {0,1,2,3}{
    \ifnum\i=\j
      \draw[draw=\c!75!black,fill=\c!72,line width=0.9pt,rounded corners=0.5pt]
        ({5.52+0.49*\i},{-1.32-0.45*\j}) rectangle ++(0.44,0.38);
    \else
      \draw[dashed,draw=\c!82!black,fill=\c!11,line width=0.55pt,rounded corners=0.5pt]
        ({5.52+0.49*\i},{-1.32-0.45*\j}) rectangle ++(0.44,0.38);
    \fi
  }
  \draw[flow,line width=0.55pt] (7.64,{-1.13-0.45*\j})--(8.05,{-1.13-0.45*\j});
  \node[circle,draw=ink!65,fill=white,inner sep=0pt,minimum size=0.40cm,
        font=\scriptsize] at (8.32,{-1.13-0.45*\j}) {$\Sigma$};
  \draw[flow,line width=0.55pt] (8.57,{-1.13-0.45*\j})--(9.20,{-1.13-0.45*\j});
  \draw[draw=\c!75!black,fill=\c!58,line width=0.65pt,rounded corners=0.5pt]
        (9.32,{-1.32-0.45*\j}) rectangle ++(0.47,0.38);
}
\foreach \i in {0,1,2,3}{
  \node[tinytext] at ({5.74+0.49*\i},-0.79) {$\i$};
}
\node[tinytext] at (8.33,-0.79) {FP};
\node[tinytext] at (9.55,-0.79) {$Q_4$};

\draw[flow] (9.92,-1.79)--(12.38,-1.79);
\node[font=\scriptsize\bfseries,fill=white,inner sep=1pt] at (11.13,-1.48) {Gather};
\node[tinytext,align=center] at (11.13,-2.17)
  {$DQ_4$ on all GPUs};

\foreach \i in {0,1,2,3}{
  \node[tinytext,anchor=east] at (13.54,{-1.13-0.45*\i}) {GPU $\i$};
  \foreach \j/\c in {0/segzero,1/segone,2/segtwo,3/segthree}{
    \draw[draw=\c!70!black,fill=\c!40,line width=0.5pt,rounded corners=0.5pt]
      ({13.65+0.51*\j},{-1.32-0.45*\i}) rectangle ++(0.46,0.38);
  }
}
\draw[draw=grid,fill=soft,rounded corners=2pt] (0.08,-3.35) rectangle (17.47,-2.99);
\draw[draw=segzero!75!black,fill=segzero!72,line width=0.7pt]
  (0.27,-3.24) rectangle ++(0.22,0.16);
\node[tinytext,anchor=west] at (0.58,-3.16) {local FP segment};
\draw[dashed,draw=segzero!82!black,fill=segzero!11,line width=0.55pt]
  (3.45,-3.24) rectangle ++(0.22,0.16);
\node[tinytext,anchor=west] at (3.76,-3.16) {remote 4-bit code};
\node[font=\scriptsize\bfseries,text=ink,anchor=west] at (7.20,-3.16)
  {At most 2 quantizations / element, independent of $P$};

\node[heading,anchor=west] at (0,-3.79) {(b) Tiled four-stage pipeline};
\node[tinytext,anchor=east] at (17.55,-3.79) {time $\rightarrow$ (schematic slots)};
\draw[grid,line width=0.45pt] (0,-4.00)--(17.55,-4.00);

\def\sx{5.42}\def\step{1.54}
\foreach \k in {0,1,2,3,4,5,6}{
  \node[tinytext] at ({\sx+\step*\k},-4.29) {$\k$};
}
\begin{scope}[on background layer]
  \draw[draw=tileaccent!24,fill=tileaccent!6,rounded corners=2pt]
     ({\sx+3*\step-0.71},-6.79) rectangle ({\sx+3*\step+0.71},-4.47);
\end{scope}
\foreach \r/\stage/\detail in {
  0/Push/$Q_4$ + P2P write,
  1/Reduce/$DQ_4$ + FP sum + $Q_4$,
  2/Gather/P2P share + owner $DQ_4$,
  3/Tail/peer $DQ_4$
}{
  \node[heading,anchor=west] at (0.34,{-4.80-0.60*\r}) {\stage};
  \node[tinytext,anchor=west] at (1.68,{-4.80-0.60*\r}) {\detail};
  \foreach \t in {0,1,2,3}{
    \ifnum\t=0
      \node[firstslot] at ({\sx+\step*(\t+\r)},{-4.80-0.60*\r}) {$t_{\t}$};
    \else
      \node[slot] at ({\sx+\step*(\t+\r)},{-4.80-0.60*\r}) {$t_{\t}$};
    \fi
  }
}
\foreach \r in {0,1,2}{
  \pgfmathtruncatemacro{\nr}{\r+1}
  \draw[->,draw=tileaccent,line width=0.75pt]
    ({\sx+\step*\r+0.54},{-4.80-0.60*\r-0.15})
    -- ({\sx+\step*\nr-0.54},{-4.80-0.60*\nr+0.15});
}
\node[tinytext,anchor=west] at (0.34,-7.19)
  {Cross-GPU readiness: peer-visible flags $\quad\vert\quad$ Within-GPU ordering: CUDA events};
\node[font=\scriptsize\bfseries,text=tileaccent!65!black,anchor=east] at (17.45,-7.19)
  {4 tiles active concurrently};
\end{tikzpicture}%
}
\caption{\textbf{The quantized P2P all-reduce.}
(a) For $P=4$, each GPU sends only its non-owned segments as 4-bit codes to their owners. Each owner dequantizes the received codes and sums them with its own unquantized segment in high precision. The reduced segment is quantized once for the all-gather, and every GPU, including the owner, dequantizes the same code to obtain identical outputs. Each element is quantized at most twice, independently of $P$.
(b) Tiles are staggered across push, reduce, gather, and tail. Peer-visible flags signal cross-GPU readiness and CUDA events order work within a GPU, allowing quantization and transfers of different tiles to overlap.}
\label{fig:decomposition}
\end{figure*}

\begin{table}[t]
\begin{minipage}[c]{0.370\linewidth}
\caption{\textbf{Quantized P2P all-reduce against a full-precision ring
all-reduce} on $\TP$ GPUs: how many times each element is (re)quantized, the
source of the injected noise $\injnoise_{\site}$, and the bytes on the wire,
with $\bits$ the communication bit width.}\label{tab:vs-ring}
\end{minipage}\hfill
\begin{minipage}[c]{0.590\linewidth}\centering
\footnotesize
\setlength{\tabcolsep}{4.5pt}
\begin{adjustbox}{max width=\linewidth}
\begin{tabular}{@{}lccc@{}}
\toprule
 & \textbf{Where the sum} & \textbf{Quant.\ events} & \textbf{Wire} \\
 & \textbf{is formed}     & \textbf{per element}    & \textbf{format} \\
\midrule
16-bit ring AR        & low or high prec.       & $0$            & $16$-bit \\
Naive low-bit ring  & (invalid in low-bit)    & $\order{\TP}$  & $\bits$-bit \\
\textbf{Ours (P2P)} & owner, high prec.       & $\mathbf{2}$   & $\bits$-bit \\
\bottomrule
\end{tabular}
\end{adjustbox}
\end{minipage}
\end{table}

\subsection{Pipelined execution and synchronization}\label{subsec:pipeline}

The all-reduce lies between two dependent matrix multiplications, so its data
path cannot overlap with them, and the room for overlap on this path is inside
the primitive, between its quantization work and its data transfers.
To expose it, we split each segment into small contiguous blocks, called
\emph{tiles}, and run the tiles through a four-stage pipeline. The stages are push,
reduce, gather, and tail. Push quantizes the remote segments and writes them to
their owners. Reduce dequantizes the received copies, adds them to the owned
segment, and requantizes the result. Gather sends this requantized segment to the
other GPUs and, on the owning GPU, dequantizes it to produce that GPU's own portion
of the output. Tail dequantizes the reduced segments arriving from the other
owners. Because the tiles pass through the stages independently, one tile can be
quantized while an earlier tile is still traveling over NVLink
(Fig.~\ref{fig:decomposition}(b)), so the quantization work of one tile is hidden
behind the transfers of the tiles around it. The stream and event schedule is given
in Appendix~\ref{app:abi}.

The four stages run on different GPUs and must be kept in order across them.
Routing that ordering through the host CPU would add latency and serialize the
pipeline, so we synchronize the GPUs directly over their peer-to-peer (P2P) links.
Each GPU keeps a small flag in memory that its peers can read over the P2P link. A
stage advances the flag after it finishes a tile, and the consuming stage on
another GPU waits until the flag reaches the value it expects before it starts.
Ordering within a GPU is handled by CUDA events, and ordering across GPUs by these
flags.

\looseness=-1 Signaling after each tile, rather than after the whole tensor, lets communication
and quantization use different block sizes, which they need because they favor
opposite choices. Peer-to-peer transfers reach their peak bandwidth with a few
large blocks, whereas the quantization kernels are most efficient with many small
blocks. An implementation that exchanges the tensor only after it has been staged
in full must commit to a single block size for both. Tracking
readiness at the tile level removes this coupling, because a consuming stage
waits only for the specific tiles it needs rather than for the entire tensor. The exact
flag values are given in Appendix~\ref{app:abi}.

\subsection{Quantization}\label{subsec:quant}

The quantizer must not overflow the $\bits$-bit codes and must preserve model
quality. We quantize group-wise. Each group of $\gsz$ consecutive values is sent
as $\bits$-bit codes together with a small set of per-group parameters.

A symmetric group tracks its absolute maximum and uses a single scale $s$,
\begin{equation}\label{eq:qsym}
\begin{aligned}
  q(x) &= \mathrm{clip}\!\big(\mathrm{round}(x/s),\,-(2^{\bits-1}\!-\!1),\,2^{\bits-1}\!-\!1\big),\\[2pt]
  s &= \max\nolimits_k |x_k| \,\big/\, (2^{\bits-1}-1).
\end{aligned}
\end{equation}
The scale sends the largest magnitude to the code $2^{\bits-1}\!-\!1$, so the most
negative code $-2^{\bits-1}$ is never used. Leaving it out keeps the levels
symmetric about zero, which we found to be more stable at 4~bits.

Activations are often not centered on zero, so a symmetric grid, whose levels are
symmetric about zero, spends much of its range on values that never appear. An
asymmetric group instead fits an affine map to the observed range. This map has
strictly more freedom, so it should be at least as accurate as the symmetric
grid. At 4~bits a naive asymmetric map often is not, and the reason lies in its
zero-point. The map is set by a scale $s$ and an integer zero-point $z$,
\begin{equation}\label{eq:qasym}
\begin{aligned}
  s &= \frac{\max_k x_k - \min_k x_k}{2^{\bits}-1},\\[2pt]
  q(x) &= \mathrm{clip}\!\big(\mathrm{round}(x/s) + z,\,0,\,2^{\bits}-1\big),
\end{aligned}
\end{equation}
and dequantizes as $\hat{x}=s\,(q-z)$. The integer $z$ is the code that maps to
$0$, so $0$ is represented exactly. A naive map instead anchors the grid at the
group minimum, leaving the level nearest $0$ at the non-integer position
$-\min_k x_k/s$, so $0$ falls between two codes. Since activation distributions
are sharply peaked at zero, this rounds the densest part of the distribution, and
with only sixteen levels the error there outweighs the finer resolution the
asymmetric map was meant to gain.

\looseness=-1 Our policy is zero-aware. When a group straddles zero we round its
zero-point to the nearest integer. When a group lies entirely on one side of zero
the value $0$ never occurs, so we keep the naive map and spend the whole grid on
the values that are present. Groups in between blend smoothly between the two. This change alone removes
the 4-bit regression, and
asymmetric quantization then outperforms the symmetric
policy (Sec.~\ref{sec:experiments}).

\looseness=-1 We expose only the two parameters that matter, the bit width $\bits\in\{4,8\}$ and
the group size $\gsz$, and otherwise keep both policies fixed. The per-group storage is
compact, one value for symmetric and a scale and a zero-point for
asymmetric, and a backward-pass scaling prevents small gradient groups
from underflowing it (Sec.~\ref{subsec:backward-impl}).

\section{Dual Error Feedback}\label{sec:dual-ef}

The quantized all-reduce injects the error $\injnoise_k$ of Eqn.~\eqref{eq:qar} at
every site, with $\norm{\injnoise_k}\le\qbound$. Left uncorrected, these errors
accumulate, and the partial sum $\mathbf{S}_t$ of Eqn.~\eqref{eq:partial-sum}, which
governs the output error, grows to order $\Nsite\qbound$ over a pass
(Sec.~\ref{subsec:controlling}). We remove this growth with an error-feedback
compensator that realizes the correction $\mathbf{C}_k$ of Eqn.~\eqref{eq:net-error}.
It recovers each injected error, summed over the group, from the output of the
all-reduce and feeds a running estimate of it back into the next all-reduce, so the injected errors cancel in
$\mathbf{S}_t$ rather than accumulate. Section~\ref{subsec:compensator} defines the
compensator and Sec.~\ref{subsec:suppression} shows, term by term, why it keeps
$\mathbf{S}_t$ bounded.

\subsection{Two-timescale error-feedback compensator}\label{subsec:compensator}

On each GPU the compensator keeps two moving averages, applies a blend of them as a
correction before each all-reduce, and refreshes them from the residual it then
forms. The update at site $\site$ on GPU $i$ runs as follows.

It first forms a correction from its current averages and adds it to the
contribution, before the all-reduce quantizes it,
\begin{equation}\label{eq:apply}
  \mathbf{c}_{k,i}=\eta\big(a\,\mathbf{m}^{(1)}_{k-1,i}+(1-a)\,\mathbf{m}^{(2)}_{k-1,i}\big),
  \qquad
  \mathbf{x}'_{k,i}=\mathbf{x}_{k,i}+\mathbf{c}_{k,i},
\end{equation}
with blend weight $a\in[0,1]$ and gain $\eta\in[0,1]$. The all-reduce acts on the
compensated contribution $\mathbf{x}'_{k,i}$ and returns
$\hat{\mathbf{y}}_k=\sum_i\mathbf{x}'_{k,i}+\injnoise_k$ (Eqn.~\eqref{eq:qar}). From it
each GPU $i$ forms the residual
$\mathbf{u}_{k,i}=\mathbf{x}'_{k,i}-\hat{\mathbf{y}}_k/\TP$ of
Sec.~\ref{subsec:controlling}, whose two operands are already on the device, so it
costs no communication, and whose group sum recovers the injected error,
$\sum_i\mathbf{u}_{k,i}=-\injnoise_k$ (Eqn.~\eqref{eq:residual-identity}). The two
averages are then refreshed with this residual, at decay rates
$\beta_1,\beta_2\in[0,1)$,
\begin{equation}\label{eq:ema}
  \mathbf{m}^{(j)}_{k,i}=\beta_j\,\mathbf{m}^{(j)}_{k-1,i}+(1-\beta_j)\,\mathbf{u}_{k,i},
  \qquad j=1,2.
\end{equation}
Equation~\eqref{eq:apply} is the error-feedback step. Summed over the group, the
averages estimate the errors injected at earlier sites, and feeding them back into
the input corrects those errors rather than passing them on. Each GPU stores and updates its own averages, so the
compensator adds no communication, and summed over the group the corrections give the
aggregate $\mathbf{C}_k=\sum_i\mathbf{c}_{k,i}$ of Eqn.~\eqref{eq:net-error}. The
averages are reset to zero at the start of each pass. One pair serves all sites of
a pass, so the state costs two tensors of the size of the contribution per pass
direction. Algorithm~\ref{alg:dual} lists
the update.

\begin{algorithm}[t]
\caption{Compensated all-reduce at site $\site$ (GPU $i$)}\label{alg:dual}
\small
\DontPrintSemicolon
\KwIn{contribution $\mathbf{x}_{k,i}$, state $\mathbf{m}^{(1)},\mathbf{m}^{(2)}$, and parameters $\beta_1,\beta_2,a,\eta$}
\KwOut{block output $\hat{\mathbf{y}}_k$}
\BlankLine
$\mathbf{c}\leftarrow\eta\big(a\,\mathbf{m}^{(1)}+(1-a)\,\mathbf{m}^{(2)}\big)$\tcp*{correction}
$\mathbf{x}'\leftarrow\mathbf{x}_{k,i}+\mathbf{c}$\tcp*{apply}
$\hat{\mathbf{y}}_k\leftarrow\QAR\big(\{\mathbf{x}'_{k,j}\}_j\big)$\tcp*{all-reduce}
$\mathbf{u}\leftarrow\mathbf{x}'-\hat{\mathbf{y}}_k/\TP$\tcp*{residual}
$\mathbf{m}^{(1)}\leftarrow\beta_1\mathbf{m}^{(1)}+(1-\beta_1)\,\mathbf{u}$\tcp*{update}
$\mathbf{m}^{(2)}\leftarrow\beta_2\mathbf{m}^{(2)}+(1-\beta_2)\,\mathbf{u}$\tcp*{update}
\end{algorithm}

\subsection{Suppression of error accumulation}\label{subsec:suppression}

We first write the correction in closed form. Unrolling Eqn.~\eqref{eq:ema} from the
reset gives $\mathbf{m}^{(j)}_{k,i}=(1-\beta_j)\sum_{s\le k}\beta_j^{\,k-s}\mathbf{u}_{s,i}$,
so each average is a decaying sum of past residuals. Substituting into
Eqn.~\eqref{eq:apply}, summing over the group, and using
$\sum_i\mathbf{u}_{s,i}=-\injnoise_s$ gives the correction as a weighted sum of the
earlier injected errors,
\begin{equation}\label{eq:kernel}
\begin{aligned}
  \mathbf{C}_k&=-\eta\sum_{s<k}w_{k-1-s}\,\injnoise_s,\\
  w_r&:=a(1-\beta_1)\beta_1^{\,r}+(1-a)(1-\beta_2)\beta_2^{\,r},
\end{aligned}
\end{equation}
where $w_r$ is the weight placed on an error of age $r$, and the weights sum to one,
$\sum_{r\ge0}w_r=1$.

We can now read off how much of each injected error survives in the partial sum.
Putting $\bm{\rho}_k=\injnoise_k+\mathbf{C}_k$ into
$\mathbf{S}_t=\sum_{k\le t}\bm{\rho}_k$ (Eqn.~\eqref{eq:partial-sum}) and collecting the
terms in each $\injnoise_s$,
\begin{equation}\label{eq:coeff}
  \mathbf{S}_t=\sum_{s\le t}\phi_{t-s}\,\injnoise_s,
  \qquad
  \phi_m:=1-\eta\sum_{r=0}^{m-1}w_r.
\end{equation}
The coefficient $\phi_m$ is the fraction of an injected error that still remains in the
partial sum $m$ sites after it enters. It starts at $\phi_0=1$ and decreases with age,
since $\sum_{r=0}^{m-1}w_r$ climbs toward $\sum_{r\ge0}w_r=1$. At gain $\eta=1$ it
decreases to zero. An old error is then subtracted away and eventually leaves
$\mathbf{S}_t$, so the partial sum carries only a bounded window of recent errors.
As a result, $\max_t\norm{\mathbf{S}_t}$ stays bounded as the pass lengthens, in place of the
order-$\Nsite\qbound$ growth of the uncorrected sum $\sum_{k\le t}\injnoise_k$
(Sec.~\ref{subsec:controlling}). Section~\ref{sec:theory} turns
Eqn.~\eqref{eq:coeff} into the exact bound, and treats the case $\eta<1$, where
$\phi_m$ settles at $1-\eta$ and a fixed share of each error persists.

The two averages, with decay rates $\beta_1$ and $\beta_2$, give the correction two
time scales. A single average with decay $\beta$ weights past residuals by
$(1-\beta)\beta^r$, which shrink by the factor $\beta$ each site, so it
spans about $1/(1-\beta)$ sites. The decay rate $\beta_1$ sets a longer memory of
about $1/(1-\beta_1)$ sites, and $\beta_2$ sets a shorter memory of about
$1/(1-\beta_2)$ sites. The longer one estimates the part of the injected error that
stays similar from site to site, and the shorter one adapts to the part that changes
quickly. One average alone cannot do both. Combining moving averages at several time scales in this way
is a known technique for accelerating first-order
optimization~\citep{ademamix,xie2022adan}. The two decay rates
and the blend $a$ are studied in Sec.~\ref{sec:experiments}.

\subsection{Fusing the compensator into the all-reduce}\label{subsec:fusion}

The compensator of Sec.~\ref{subsec:compensator} runs two steps at every
site, the apply step $\*x'=\*x+\*c$ of Eqn.~\eqref{eq:apply} before the
all-reduce and, after it returns $\hat{\*y}$, the update step of
Eqn.~\eqref{eq:ema} on the residual $\*u=\*x'-\hat{\*y}/\TP$. Here and in
the rest of this section we drop the site and GPU indices. Each step reads and
writes tensors as large as $\*x$ itself, so implemented as separate kernels the
two steps add full memory passes of their own. This subsection shows that the apply step can instead run inside the
entry pass of the quantized all-reduce of Sec.~\ref{sec:quant-allreduce} and the
update beside the next matrix multiplication, at a small added latency
(Table~\ref{tab:latency}).
Section~\ref{subsec:backward-impl} then describes the deployment of the
compensator on the backward pass.

A direct implementation surrounds the all-reduce with two kernels, an apply
kernel that forms $\*x'=\*x+\*c$ and writes it to a buffer of its own, and an
update kernel that forms $\*u=\*x'-\hat{\*y}/\TP$ and refreshes $\*m^{(1)}$,
$\*m^{(2)}$. The two kernels are costly in two ways. They move about ten
tensor-sized arrays through memory, while the all-reduce itself reads $\*x$
once and writes $\hat{\*y}$ once. They also stretch the critical
path: the apply kernel must deliver $\*x'$ before the all-reduce can start, and
the update kernel needs $\hat{\*y}$, so it starts only after the all-reduce
returns. The all-reduce is bandwidth bound, so the extra traffic cannot hide
behind computation, and the two effects together place the full cost of
compensation in the latency of every site.

On entry the all-reduce already moves the full tensor through registers. Every
element of $\*x$ is loaded, the quantization statistics are formed, and the value
is quantized for transport (Sec.~\ref{subsec:pipeline}). The apply step needs $\*x$, which
the entry pass is holding. The update step needs $\hat{\*y}$, which the next
matrix multiplication only reads, so the update can run while that
multiplication does.

The apply step is therefore fused into the entry pass. As each element of $\*x$
arrives in registers, the kernel also loads $\*m^{(1)}$ and $\*m^{(2)}$, forms
$\*c=\eta\,(a\,\*m^{(1)}+(1-a)\,\*m^{(2)})$ and $\*x'=\*x+\*c$ in registers, and
only then computes the group statistics and quantizes. The tensor $\*x'$ exists
only in registers and is never written to memory. This is also the only place
where the step can go. The all-reduce quantizes $\*x'$, not $\*x$
(Sec.~\ref{subsec:compensator}), so the quantization statistics must be taken on
$\*x'$, and the entry load is the last moment before they are formed.

The update step is taken off the critical path. Once the exit pass has written
$\hat{\*y}$, the update runs on a separate stream alongside the next matrix
multiplication. It recomputes $\*x'=\*x+\*c$ from $\*x$, $\*m^{(1)}$,
$\*m^{(2)}$, forms $\*u=\*x'-\hat{\*y}/\TP$, and applies
$\*m^{(j)}\leftarrow\beta_j\,\*m^{(j)}+(1-\beta_j)\,\*u$ in place. The apply step
of the next site waits for it, so it always reads current averages. Every GPU, owner included, dequantizes $\hat{\*y}$ from the same
transported code (Sec.~\ref{subsec:cannot-sum}), so the residuals across the
group are formed from a common $\hat{\*y}$, as the identity
$\sum_i\*u_i=-\injnoise$ of Eqn.~\eqref{eq:residual-identity} requires.

\looseness=-1 After fusion the apply kernel and the $\*x'$ buffer are gone, and
neither step runs as a serial phase on the critical path. The apply step
inherits the pipelining of Sec.~\ref{subsec:pipeline} and runs tile by tile
inside the entry stage, overlapped with the communication of the other tiles,
and the update runs beside the next matrix multiplication. The extra traffic
falls from about ten tensor-sized arrays to about eight, reading $\*m^{(1)}$,
$\*m^{(2)}$ at entry, reading $\*x$, $\hat{\*y}$, $\*m^{(1)}$, $\*m^{(2)}$ in the
update, and writing $\*m^{(1)}$, $\*m^{(2)}$ back. Six of the eight are the cost
any implementation pays to consult the averages and refresh them, and the update
traffic overlaps with the matrix multiplication. The added
arithmetic hides behind the memory accesses of the entry pass and of the update, and
Table~\ref{tab:latency} measures the resulting cost. Fusion changes when and
where each step executes,
not what it computes, and $\*m^{(1)}$, $\*m^{(2)}$ remain the only state,
updated in place.

\subsection{The backward pass}\label{subsec:backward-impl}

Training runs the all-reduce in both directions. The compensator is unchanged on
the backward pass, with the same recurrence and the same coefficients, but its
working conditions change in two ways.

First, backward errors enter the computation graph at a different point. A
forward site delivers its output directly to the residual addition, while the
output of a backward site reaches the gradient stream through the Jacobian of
the normalization (Sec.~\ref{subsec:noise-process}), the setting in which
Appendix~\ref{app:proofs} proves the backward source bound. Residuals of the
two directions therefore have different coordinates and different
magnitudes, and a single pair $\*m^{(1)}$, $\*m^{(2)}$ fed by both streams
would form a correction that matches neither pass. Each direction keeps its
own pair and runs the recurrence of Sec.~\ref{subsec:compensator} along its
own sequence of sites.

Second, gradients are small enough to underflow the storage of the quantizer's
scale. The quantizer of Sec.~\ref{subsec:quant} stores one scale $s$ per
group in half precision. Backward gradients are orders of magnitude
smaller than activations, so a group scale can fall below the smallest
positive half-precision number and load back as zero. Every value in that
group then dequantizes to zero. The remedy is a change of units around the
communication. Before the quantized all-reduce, every GPU multiplies its
gradient by the same power-of-two factor, chosen from the largest gradient
magnitude across the tensor-parallel group with a one-scalar all-reduce. After the all-reduce, it divides the result
by that factor. The two steps of the compensator run outside the scaling, so
$\*m^{(1)}$, $\*m^{(2)}$ stay in original gradient units. The scaling is
exact and leaves the analysis of Secs.~\ref{subsec:compensator}
and~\ref{sec:theory} unchanged, as Appendix~\ref{app:abi} verifies.

\section{Theoretical Guarantees}\label{sec:theory}

\subsection{Decomposition of the output error}\label{subsec:decomposition}

The quantized all-reduce injects an error at each of the $\Nsite$ sites of a pass.
Let $\mathbf{z}_k$ be the hidden state carried by the exact all-reduce and
$\hat{\mathbf{z}}_k$ the hidden state carried by the quantized one, so that
\begin{equation}\label{eq:two-rec}
  \mathbf{z}_{k+1}=\mathbf{z}_k+F_k(\mathbf{z}_k),
  \qquad
  \hat{\mathbf{z}}_{k+1}=\hat{\mathbf{z}}_k+F_k(\hat{\mathbf{z}}_k)+\bm{\rho}_k,
\end{equation}
where $F_k$ is the block at site $k$, the map that normalizes the hidden state and
applies the attention or MLP transformation, and
$\bm{\rho}_k=\injnoise_k+\mathbf{C}_k$ is the net error of Eqn.~\eqref{eq:net-error}.
Write $\mathbf{E}_k=\hat{\mathbf{z}}_k-\mathbf{z}_k$ for the deviation between the
two, with $\mathbf{E}_0=\mathbf{0}$. Subtracting the recursions of
Eqn.~\eqref{eq:two-rec},
\begin{equation}\label{eq:err-exact}
  \mathbf{E}_{k+1}=\mathbf{E}_k+\bm{\psi}_k+\bm{\rho}_k,
  \qquad
  \bm{\psi}_k=F_k(\hat{\mathbf{z}}_k)-F_k(\mathbf{z}_k).
\end{equation}
Equation~\eqref{eq:err-exact} is exact, and it separates the two roles of a site. The
term $\bm{\psi}_k$ propagates the deviation accumulated so far, and adds nothing of its
own, since $\mathbf{E}_k$ determines it. The term $\bm{\rho}_k$ injects the new error
of site $k$. Expanding $F_k$ about $\mathbf{z}_k$,
\begin{equation}\label{eq:err-expand}
  \bm{\psi}_k=\mathbf{J}_k\mathbf{E}_k+\mathbf{r}_k,
  \qquad
  \norm{\mathbf{r}_k}\le\tfrac12 H_k\norm{\mathbf{E}_k}^2 ,
\end{equation}
with $\mathbf{J}_k$ the Jacobian of $F_k$ at $\mathbf{z}_k$ and $H_k$ a bound on its
second derivative on the segment from $\mathbf{z}_k$ to $\hat{\mathbf{z}}_k$. Substituting Eqn.~\eqref{eq:err-expand} into
Eqn.~\eqref{eq:err-exact} and unrolling from $\mathbf{E}_0=\mathbf{0}$,
\begin{equation}\label{eq:carry}
\begin{aligned}
  \mathbf{E}_{\Nsite}&=\sum_{t=0}^{\Nsite-1}\bm{\Phi}_{\Nsite,t+1}\big(\bm{\rho}_t+\mathbf{r}_t\big),\\
  \bm{\Phi}_{\Nsite,j}&=(\mathbf{I}+\mathbf{J}_{\Nsite-1})\cdots(\mathbf{I}+\mathbf{J}_{j}).
\end{aligned}
\end{equation}
The deviation at the output is a sum of the injected errors, each propagated there by
the Jacobians of the sites that follow it, together with the second-order remainders.
The remainders are second order in the injected errors, and
Proposition~\ref{prop:remainder} in Appendix~\ref{app:proofs} quantifies them. The first-order terms therefore carry the analysis, and two
quantities remain to be bounded, the propagation $\bm{\Phi}$ and the accumulation of
the $\bm{\rho}_t$.

The accumulation must be measured through the partial sum $\mathbf{S}_t$ of
Eqn.~\eqref{eq:partial-sum}, not through the individual errors. A term-by-term
triangle inequality on Eqn.~\eqref{eq:carry} allows $\norm{\bm{\rho}_k}$ as
large as $2\qbound$ and gives a bound of order $\Nsite\qbound$. This bound
shows no benefit from error feedback, because feedback does not reduce any single
net error. Instead, feedback makes the errors cancel as they accumulate
(Sec.~\ref{subsec:controlling}). Summation by parts
rewrites the first-order part of Eqn.~\eqref{eq:carry} so that the errors enter
only through their partial sums, with coefficients formed from the propagation
(Appendix~\ref{app:proofs}).

Section~\ref{subsec:bounds} bounds the two in turn, the partial sums by a
constant for any error sequence, and the propagation in mean square under one
assumption on the trained network. Neither result restates existing error
feedback theory. That theory concerns gradient compression, where the compressed
errors perturb an optimization step and the outcome is a convergence rate. Here the
errors enter the forward computation and are then carried through the network itself.
This propagation has no counterpart in the gradient setting, and
Assumption~\ref{asm:noreinf} and Theorem~\ref{thm:prop} address it.

\subsection{Depth-independent bounds}\label{subsec:bounds}

This subsection contains two results. Theorem~\ref{thm:source} bounds the source. With
compensation, the accumulated error at the sites stays below a constant,
about $2\qbound$ at gain $\eta=1$, for any
error sequence, while without compensation it can reach
$\Nsite\qbound$. This requires no probabilistic model and no property of the network.
Theorem~\ref{thm:prop} bounds the output. Under a noise model for the injected errors
and one assumption on the trained network, the mean-square first-order deviation of
the output, measured against the energy of the hidden state, stays at the level set
by a single rounding error, at every depth. Theorem~\ref{thm:prop} builds on
Theorem~\ref{thm:source}, whose constant reappears inside its bound. The new element
is the propagation, which Assumption~\ref{asm:noreinf} controls.

We start with the source. At gain $\eta=1$ the corrections return each error in full. Of an
error injected $m$ sites ago, the fraction $\sum_{r<m}w_r$ has been subtracted back
out by later corrections, and the fraction
\begin{equation}\label{eq:tail}
  T_m:=\sum_{r\ge m}w_r
\end{equation}
remains. These tails are exactly the coefficients $\phi_m$ of Eqn.~\eqref{eq:coeff},
so each error enters the partial sum with weight $1$ and decays toward zero, and no
partial sum can exceed the total of the tails.

\begin{theorem}[source bound]\label{thm:source}
At gain $\eta=1$, for every error sequence with $\norm{\injnoise_k}\le\qbound$,
\begin{equation}\label{eq:source-bound}
\begin{aligned}
  \max_t\norm{\mathbf{S}_t}\;&\le\;\qbound\sum_{m\ge0}T_m\;=\;(1+\mu_w)\,\qbound,\\
  \mu_w\;&=\;\textstyle\sum_{r\ge0}r\,w_r,
\end{aligned}
\end{equation}
where $\mu_w$ is the mean age, in sites, at which the compensator returns an injected
error. Without compensation the same quantity reaches $\Nsite\qbound$.
\end{theorem}

The theorem has one hypothesis, the bound $\qbound$ on each error. The $T_m$ are
nonnegative, so the proof is a triangle inequality, and it covers errors that are
dependent, correlated with the activations, or chosen adversarially. The constant
reflects how the compensator works. An injected error stays in the partial sum until
later corrections have subtracted it back out, which takes $1+\mu_w$ sites on
average, about $2.1$ for the decay rates and blend used in the
experiments. At any moment the partial
sum therefore holds only the errors of the last few sites, about two errors in total.
Without compensation errors are never subtracted, all $\Nsite$ of them accumulate, and
the partial sum reaches $\Nsite\qbound$. Compensation reduces the accumulated source
from $\Nsite\qbound$ to $(1+\mu_w)\qbound$, a factor of $\Nsite/(1+\mu_w)$,
about $27$ at the $\Nsite=56$ of Sec.~\ref{subsec:noise-process}.

Theorem~\ref{thm:source} controls what the sites accumulate. The deviation at the
output also depends on the factor $\bm{\Phi}$ of Eqn.~\eqref{eq:carry}, the
transformation the network applies to each error on its way to the output. Bounds that
take the worst case of every site grow with depth. The standard estimate is
$\norm{\bm{\Phi}}\le(1+L)^{\Nsite}$, with $L$ a Lipschitz constant of the blocks.
Growth with depth is not an artifact of the estimate. Proposition~\ref{prop:sharp}(a) in Appendix~\ref{app:proofs} constructs a
network that amplifies one fixed direction at every site, and in it even
isotropic injections produce a deviation that grows exponentially in
$\sqrt\Nsite$. A depth-independent bound must therefore exclude networks that
keep amplifying one direction. The noise model rules out injections aligned
with such a direction, and Assumption~\ref{asm:noreinf} rules out the networks. These are the two hypotheses of
Theorem~\ref{thm:prop}.

We first treat the injections. A rounding error is fixed by the quantizer at the moment of
injection, from the activations alone, with no reference to the directions the network
later amplifies. We model the errors accordingly. Conditioned on everything injected
earlier, $\injnoise_s$ has mean zero, no preferred direction, and norm at most
$\qbound$~\citep{quantnoise}. Under this model the contributions of different sites are
uncorrelated, and the linearized deviation $\mathbf{E}_\Nsite^{\mathrm{lin}}$, the
first-order part of Eqn.~\eqref{eq:carry}, satisfies
\begin{equation}\label{eq:meansq}
  \mathbb{E}\,\bigl\|\mathbf{E}_\Nsite^{\mathrm{lin}}\bigr\|^2\;\le\;\frac{\qbound^2}{d}\sum_s\norm{\mathbf{K}_s}_F^2,
\end{equation}
where $d$ is the dimension of the hidden state, $\mathbf{K}_s$ is the matrix carrying
$\injnoise_s$ to the output, built from the tails $T_m$ and the propagation
$\bm{\Phi}$ (Appendix~\ref{app:proofs}), and $\norm{\cdot}_F$ is the Frobenius norm.

We next treat the networks. The reference scale here is the hidden state itself. Let
$\sigma_k^2=\norm{\mathbf{z}_k}^2/d$, the mean-square entry of the hidden state at
site $k$. In a pre-normalization transformer $\sigma_k$ rises with depth, because
every site adds new content to the hidden state~\citep{curseofdepth}. The
hidden state therefore gains energy at a definite rate, and
Assumption~\ref{asm:noreinf} allows a deviation to gain energy at most at that same
rate.

\begin{assumption}[bounded reinforcement]\label{asm:noreinf}
There are $\varepsilon_k\ge0$ with $\sum_k\varepsilon_k\le\Lambda$, for a $\Lambda$ independent of $\Nsite$, such that
for every site $k$, every earlier site $s$, and every isotropic injection $\injnoise_s$ from the
noise model, the deviation $\mathbf{v}=\bm{\Phi}_{k,s+1}\injnoise_s$ it produces at
site $k$ satisfies
\[
  \mathbb{E}\,\bigl\|(\mathbf{I}+\mathbf{J}_k)\mathbf{v}\bigr\|^2
  \;\le\;(1+\varepsilon_k)\,\frac{\sigma_{k+1}^2}{\sigma_k^2}\;\mathbb{E}\norm{\mathbf{v}}^2 .
\]
Appendix~\ref{app:proofs} completes the family of deviations $\mathbf{v}$ used in the
proofs.
\end{assumption}

The assumption does not restrict what a block does to any single vector. A block may
rotate an arriving deviation, spread it across coordinates, or overwrite it. The
$\varepsilon_k$ give further room. Single sites may amplify beyond the ratio
$\sigma_{k+1}^2/\sigma_k^2$, and only their total over all sites must stay bounded
independently of depth.
The assumption excludes one behavior, a network that keeps amplifying the arriving
errors from site to site. Trained transformers show the opposite
tendency. Consecutive blocks tend to cancel each other's
contributions~\citep{tlcm}, and small perturbations grow with depth at the rate of the
hidden state, not faster~\citep{numstable}. Such an exclusion is also
necessary. The amplifying network of Appendix~\ref{app:proofs} satisfies the
noise model, so no depth-independent bound holds without an assumption of this
kind.

Write $\bar{q}=\qbound/(\sqrt d\,\sigma_1)$, the per-entry size of one injected error
relative to the scale $\sigma_1$ of the hidden state after the first site.

\begin{theorem}[propagation]\label{thm:prop}
Under Assumption~\ref{asm:noreinf}, the noise model, and the regularity
conditions of Appendix~\ref{app:proofs}, the following hold.
\begin{itemize}
\item[(i)] In mean square,
\[
  \frac{\mathbb{E}\,\|\mathbf{E}_\Nsite^{\mathrm{lin}}\|^2}{d\sigma_\Nsite^2}
  \;\le\;2\Big[\tau_2\frac{\sigma_1^2}{\sigma_\Nsite^2}+(1+\mu_w)^2C_0e^{\Lambda}\Big]\bar{q}^2,
\]
with $\tau_2=\sum_mT_m^2$ a constant of the kernel and $C_0$ a constant of
the checkpoint defined in Appendix~\ref{app:proofs}.
\item[(ii)] For every $p\in(0,1)$, with probability at least $1-p$,
\[
  \frac{\|\mathbf{E}_\Nsite^{\mathrm{lin}}\|^2}{d\sigma_\Nsite^2}
  \;\le\;C\,\bar{q}^2\log\frac2p+C'\bar{q}^2\log^2\frac2p ,
\]
with $C$ and $C'$ the depth-independent constants of
Appendix~\ref{app:proofs}.
\end{itemize}
The term $\tau_2\,\sigma_1^2/\sigma_\Nsite^2$ in (i) decays with depth. No other
quantity in (i) and (ii) grows with $\Nsite$.
\end{theorem}

The quantity bounded is the energy of the deviation measured against the energy of the
hidden state. The second-order remainders are at most $C_2\bar{q}\sqrt\Nsite$ times
the scale $\bar{q}\sqrt d\,\sigma_\Nsite$ of the linear bound, for as long as the deviation stays within twice its linear
bound (Proposition~\ref{prop:remainder}, sharp by Proposition~\ref{prop:sharp}(b), see also Corollary~\ref{cor:budget}).
Here $C_2$ depends on the curvature and the energy growth of the network and grows at most as $\log^2(\Nsite/p)$.
Up to the logarithm in $C_2$, the factor $C_2\bar{q}\sqrt\Nsite$ is small whenever
the noise injected over a pass is a small fraction of the hidden state.

In the same mean-square terms the advantage over the uncompensated primitive is
direct. With compensation the variance of the accumulated source stays
below $\tau_2\qbound^2$ at every site, about $1.3\,\qbound^2$ for the decay
rates and blend used in the experiments. Theorem~\ref{thm:prop}
turns this bounded source into a first-order output deviation at the level of a single
rounding error, independent of depth. Without compensation each site keeps its full
contribution and the same bound grows as $\log\Nsite$
(Remark~\ref{rem:nocomp}). Compensation multiplies the term of each site in the bound by at most
$2(1+\mu_w)^2C_0\Delta_k/\sigma_k^2$, where $\Delta_k=\sigma_{k+1}^2-\sigma_k^2$ is the
energy that block $k$ adds, so the sum stays bounded.

The source analysis extends to the backward pass. Its injected errors enter through the
Jacobian of the normalization (Sec.~\ref{subsec:noise-process}), and
Appendix~\ref{app:proofs} proves the corresponding depth-independent
source bound, in mean square and under a smoothness condition on the
normalization Jacobians.
Section~\ref{subsec:backward-impl} describes the backward implementation.

\subsection{Choice of the compensator weights}\label{subsec:weights}

Theorems~\ref{thm:source} and~\ref{thm:prop} fix the gain at $\eta=1$ and take the
decay rates $\beta_1,\beta_2$ of Sec.~\ref{sec:dual-ef} as given. Both choices
affect the guarantee.

The gain interpolates between no compensation and full compensation. At gain $\eta$
the coefficient of an injected error in the partial sum is $(1-\eta)+\eta T_{t-s}$,
and the source bound follows the same split,
\begin{equation}\label{eq:gain-bound}
  \max_t\norm{\mathbf{S}_t}\;\le\;(1-\eta)\,\Nsite\qbound+\eta\,(1+\mu_w)\,\qbound .
\end{equation}
Theorem~\ref{thm:source} is the case $\eta=1$, the only gain at which the growing term
vanishes. At $\eta<1$, the fraction $1-\eta$ of every error is never corrected and
accumulates as it would without compensation. The compensator state is reset at the
start of each pass, so this growth does not carry over between passes.
In mean square this uncorrected share is small. The net error is linear in the
gain, so the first-order output deviation is $1-\eta$ times the uncompensated deviation plus
$\eta$ times the deviation of Theorem~\ref{thm:prop} (Corollary~\ref{cor:gain}).
At the gain $\eta=0.75$ used in the experiments, the part that grows with depth is the
uncompensated bound scaled by $2(1-\eta)^2=1/8$, which grows only as $\log\Nsite$. At
$\Nsite=56$ the variance of the source is about $5\qbound^2$, against
$56\qbound^2$ without compensation.

The decay rates enter the source bound only through $\mu_w$, so every choice of
$\beta_1,\beta_2$ keeps the source below a constant independent of $\Nsite$.
The constant grows with the memory of the averages, and is about $2\qbound$
at the values used in the experiments. The smallest constant belongs to the limit $\beta_1,\beta_2\to0$, where
the compensator keeps only the error of the preceding site and reduces to single-step
error feedback. This optimum is a worst-case one, attained on adversarial sequences,
and practical error sequences are not adversarial. Away from that extreme, correcting
with a short-range and a long-range average together does better than correcting with
the last error alone, and Sec.~\ref{sec:experiments} reports the ablation over
$\beta_1$, $\beta_2$, and $\eta$ around the values used in the experiments.

\section{Experiments}\label{sec:experiments}

The experiments cover four parts. We first measure speed, including the latency
of the quantized all-reduce, the overhead of the compensator before and after
fusion, and end-to-end speed in inference and training
(Sec.~\ref{subsec:exp-speed}). We then evaluate inference accuracy under
4-bit communication (Sec.~\ref{subsec:exp-acc}), compare training loss
curves against the 16-bit tensor-parallel baseline
(Sec.~\ref{subsec:exp-train}), and report ablations, including the sweep
over $\beta_1$, $\beta_2$, and $\eta$ (Sec.~\ref{subsec:exp-abl}).

\smallskip\noindent\textbf{Setup.}
All experiments run on a single server with eight A100 GPUs connected by
NVLink. The inference experiments serve Llama-3-8B~\citep{llama3} and Qwen3-8B~\citep{qwen3} at $\TP=4$ and
evaluate language modeling perplexity together with standard downstream
accuracy benchmarks. The training experiments pretrain Qwen3-0.6B from scratch
at $\TP=4$ with BF16 mixed precision and sequence length 2,048. Data and
schedules are given in Sec.~\ref{subsec:exp-train}. End-to-end speed is
measured on Llama-3-8B and Qwen3-8B at $\TP=8$ (Sec.~\ref{subsec:exp-speed}). Unless stated otherwise,
communication uses the 4-bit quantized all-reduce of
Sec.~\ref{sec:quant-allreduce} with group size $\gsz=128$ and the zero-aware
asymmetric quantizer, and the compensator uses $\eta=0.75$, $a=0.35$, $\beta_1=0.7$, and
$\beta_2=0.3$ (Sec.~\ref{subsec:exp-abl}).

\subsection{Latency and end-to-end speed}\label{subsec:exp-speed}

Table~\ref{tab:latency} reports the median latency of one all-reduce at
$\TP=4$ for 16-bit messages of 64\,MB to 1\,GB per GPU. Each value is the median
of 80 iterations after 30 warmup iterations, and all rows come from the same
run. The fused compensator updates its state while the next matrix
multiplication runs, so for the fused row we time the matrix multiplications and the
all-reduce together and subtract the time of the matrix multiplications alone.
The rows are the 16-bit NCCL baseline, a two-collective quantized all-reduce assembled
from DeepSpeed's quantized all-to-all and quantized all-gather primitives in
the style of ZeRO++~\citep{wang2023zero++} at 8 and 4 bits, our all-reduce at 8
and 4 bits without compensation, and the 4-bit all-reduce with the compensator
attached in two forms, as separate PyTorch operations and fused as described in
Sec.~\ref{subsec:fusion}.

\begin{table}[t]
\begin{minipage}[t]{0.490\linewidth}\centering
\caption{\textbf{Median all-reduce latency} (ms) at $\TP=4$ over message sizes,
measured in one run. Lower is better. Bold marks the minimum in each column.
The last two rows attach the compensator to our 4-bit all-reduce, unfused and
fused. The fused row updates its state beside the next matrix multiplication
and reports the time the call adds to the critical path, measured as the time of
the matrix multiplications and the all-reduce together minus that of the matrix
multiplications alone.}\label{tab:latency}
\footnotesize
\setlength{\tabcolsep}{3.5pt}
\renewcommand{\arraystretch}{1.15}
\begin{adjustbox}{max width=\linewidth}
\begin{tabular}{@{}lccccc@{}}
\toprule
 & 64\,MB & 128\,MB & 256\,MB & 512\,MB & 1\,GB \\
\midrule
NCCL 16-bit                   & 0.70 & 1.27 & 2.45 & 4.60 & 9.00 \\
Two-collective, 8-bit           & 0.74 & 1.26 & 2.27 & 4.14 & 7.87 \\
Two-collective, 4-bit           & 0.62 & 1.00 & 1.76 & 3.25 & 5.76 \\
Ours, 8-bit                 & 0.56 & 0.91 & 1.42 & 2.80 & 5.51 \\
Ours, 4-bit                 & \textbf{0.36} & \textbf{0.62} & \textbf{0.94} & \textbf{1.73} & \textbf{3.40} \\
\quad + unfused compensator & 1.27 & 2.33 & 4.50 & 8.78 & 17.49 \\
\quad + fused compensator   & 0.44 & 0.75 & 1.19 & 2.15 & 4.30 \\
\bottomrule
\end{tabular}
\end{adjustbox}
\end{minipage}\hfill
\begin{minipage}[t]{0.490\linewidth}\centering
\caption{\textbf{End-to-end speedup} ($\times$) of our fused, compensated 4-bit
all-reduce over the 16-bit NCCL all-reduce on eight A100 GPUs ($\TP=8$,
$\gsz=128$). Speedup is $T_{\mathrm{NCCL}}/T_{\mathrm{ours}}$. Inference reports
the time to first token of Llama-3-8B and Qwen3-8B, training the mean step time.
Load is the duty cycle of background peer-to-peer copies. AR slowdown is
$T_{\mathrm{AR}}^{\mathrm{NCCL}}(\ell)/T_{\mathrm{AR}}^{\mathrm{NCCL}}(0)$
for a fixed-size 1\,GB 16-bit NCCL all-reduce, the per-call message of the
batch-64 runs, at load $\ell$ against no load.}\label{tab:e2e}
\footnotesize
\setlength{\tabcolsep}{2.5pt}
\renewcommand{\arraystretch}{1.15}
\begin{adjustbox}{max width=\linewidth}
\begin{tabular*}{\linewidth}{@{\extracolsep{\fill}}lccccc@{}}
\toprule
\multirow{2}{*}{Load}
 & \multirow{2}{*}{\shortstack{AR slowdown\\($\times$)}}
 & \multicolumn{2}{c}{Inference (TTFT)}
 & \multicolumn{2}{c}{Training (step time)} \\
\cmidrule(lr){3-4}\cmidrule(l){5-6}
 & & Llama & Qwen & Llama & Qwen \\
\midrule
0\%   & 1.00 & 1.16 & 1.09 & 1.09 & 1.08 \\
50\%  & 1.80 & 1.33 & 1.21 & 1.22 & 1.20 \\
75\%  & 2.07 & 1.38 & 1.24 & 1.27 & 1.25 \\
100\% & 2.73 & 1.44 & 1.31 & 1.38 & 1.34 \\
\bottomrule
\end{tabular*}
\end{adjustbox}
\end{minipage}
\end{table}

On the 1\,GB message, our uncompensated 8-bit all-reduce is 1.63$\times$
faster than NCCL, and our uncompensated 4-bit all-reduce is 2.65$\times$ faster
and takes 37.8\% of the NCCL time. At equal bit width, our all-reduce is
1.43$\times$ faster than the two-collective baseline at 8 bits and 1.69$\times$
faster at 4 bits. The gap comes from the pipelined design of Sec.~\ref{sec:quant-allreduce}, which overlaps quantization with
communication instead of running two dependent collectives back to back.

The measured 4-bit speedup is below the $4\times$ byte ratio of 4-bit against
16-bit transport, for two reasons. First, the payload is not the only traffic.
With group size $\gsz=128$, every 128 quantized values share one pair of
half-precision quantization parameters, so a group costs 68 bytes instead of
64, which alone lowers the ideal ratio from $4\times$ to about $3.8\times$.
Second, quantization adds computation on the datapath. Each message is
quantized before it is sent and dequantized after it is received, and the owner
performs the sum in high precision between a dequantize and a requantize, none
of which exists in the 16-bit all-reduce. The pipeline hides most of this
computation behind the transfers, and the remaining gap to $3.8\times$ is the
part it cannot hide.

The two compensator rows measure the cost of compensation. Built from
separate PyTorch operations, the compensator adds two full passes over the
tensor and an intermediate buffer, and on the 1\,GB message it adds 14.09\,ms to
the all-reduce. With the apply step fused into the entry pass and the update
running beside the next matrix multiplication, it adds 0.90\,ms, 10\% of the NCCL
time. The fused path matches a PyTorch
reference implementation of the compensator exactly, with a maximum absolute
difference of zero at every size. With fused compensation, the 4-bit
all-reduce adds 4.30\,ms to the critical path for the 1\,GB message, against
9.00\,ms for NCCL, a 2.09$\times$ reduction.

Table~\ref{tab:e2e} reports end-to-end speed on the same server at $\TP=8$ with
group size $\gsz=128$. We compare the 16-bit NCCL all-reduce with our fused,
compensated 4-bit all-reduce without background load and with background
peer-to-peer copies running between the GPUs for 50\%, 75\%, and 100\% of the
time, which emulates an interconnect shared with other traffic. In each setting
both collectives run under the same load, and we report the speedup
$T_{\mathrm{NCCL}}/T_{\mathrm{ours}}$.

For inference, we measure the time to first token of Llama-3-8B and Qwen3-8B
with batch size 64 and a 2,048-token prefill. For training, we measure the mean step time of
Llama-3-8B with the same batch size and sequence length and of Qwen3-8B with
batch size 32 and sequence length 2,048. The training measurement covers the
all-reduce calls of both the forward and the backward pass, including the
gradient rescaling of the 4-bit path. Without background load, the speedup is
1.16$\times$ on Llama-3-8B and 1.09$\times$ on Qwen3-8B for inference, and
1.09$\times$ and 1.08$\times$ for training. It grows with the load and reaches
1.44$\times$ and 1.31$\times$ for inference and 1.38$\times$ and 1.34$\times$ for
training at 100\% load.

These speedups are bounded by the share of time the all-reduce blocks the
16-bit baseline. On Llama-3-8B without background load, this share is 20.3\%
of the time to first token and 18.8\% of the training step, which bounds the
speedups at 1.25$\times$ and 1.23$\times$. At 100\% load, the shares
rise to 43.0\% and 36.2\%, and the bounds to 1.75$\times$ and
1.57$\times$. A prefill waits on every
all-reduce, and our all-reduce synchronizes through flags that the GPUs set and
poll themselves, tile by tile (Sec.~\ref{subsec:pipeline}), so it starts on
each tile as soon as its data arrives and loses little time when links are
congested. At 100\% load, the fixed-size 16-bit NCCL all-reduce takes 2.73$\times$ as
long as it does without background load. The speedup thus grows with the share
of the step that the all-reduce takes.

\subsection{Inference accuracy}\label{subsec:exp-acc}

The speedups reported in Sec.~\ref{subsec:exp-speed} matter only if accuracy
is preserved. We evaluate Llama-3-8B and Qwen3-8B with the 4-bit all-reduce at
every forward site, on three groups of tasks: language modeling perplexity on
WikiText~\citep{wikitext}, accuracy on seven standard benchmarks
(PIQA~\citep{piqa}, ARC-easy, ARC-challenge~\citep{arc},
HellaSwag~\citep{hellaswag}, WinoGrande~\citep{winogrande}, MMLU~\citep{mmlu},
MMLU-Pro~\citep{mmlupro}), and reasoning on GPQA~\citep{gpqa} and
GSM8K~\citep{gsm8k}.
All measured rows, including the 16-bit baselines, are obtained by us in a single
evaluation harness~\citep{lmevalharness}, and each task uses the standard shot
count of the harness. We compare two rows on a task with a two-sided test at the
5\% level, using the binomial standard error of each accuracy. The FlashComm rows
are quoted from the original papers~\citep{flashcomm,flashcommv2} and are
set apart in Tables~\ref{tab:ppl} and~\ref{tab:downstream}.

\begin{table}[t]
\begin{minipage}[t]{0.400\linewidth}\centering
\caption{\textbf{WikiText perplexity} for Llama-3-8B and Qwen3-8B. Lower is
better. The measured rows use $\TP=4$. The FlashComm rows are quoted
from~\citep{flashcomm} and set in gray. That paper reports no 16-bit
perplexity, so these rows are compared through the change from 8 to 4 bits. Dashes
mark results not reported in the cited work.}\label{tab:ppl}
\footnotesize
\setlength{\tabcolsep}{8pt}
\renewcommand{\arraystretch}{1.15}
\begin{adjustbox}{max width=\linewidth}
\begin{tabular}{@{}lcc@{}}
\toprule
 & Llama-3-8B & Qwen3-8B \\
\midrule
16-bit                    & 6.07 & 9.70 \\
4-bit, no compensation  & 6.73 & 10.45 \\
Ours, 8-bit             & 6.07 & 9.70 \\
Ours, 4-bit             & 6.36 & 9.80 \\
\midrule
\textcolor{gray}{FlashComm 8-bit (quoted)} & \textcolor{gray}{6.14} & \textcolor{gray}{--} \\
\textcolor{gray}{FlashComm 4-bit (quoted)} & \textcolor{gray}{6.70} & \textcolor{gray}{--} \\
\bottomrule
\end{tabular}
\end{adjustbox}
\end{minipage}\hfill
\begin{minipage}[t]{0.570\linewidth}\centering
\caption{\textbf{Math reasoning accuracy} (\%) for Llama-3-8B and Qwen3-8B at
$\TP=4$. GSM8K is scored by strict-match and flexible-extract. Early-site
8-bit promotes the all-reduce calls of the first half of the layers to 8 bits,
32 calls on Llama-3-8B and 36 on Qwen3-8B, and keeps the rest at 4 bits.
The cited work
reports neither benchmark.}\label{tab:reasoning}
\footnotesize
\setlength{\tabcolsep}{3.5pt}
\renewcommand{\arraystretch}{1.15}
\begin{adjustbox}{max width=\linewidth}
\begin{tabular}{@{}llccc@{}}
\toprule
 & & & \multicolumn{2}{c}{GSM8K} \\
\cmidrule(l){4-5}
Model & Method & GPQA & strict & flexible \\
\midrule
\multirow{3}{*}{Llama-3-8B}
 & 16-bit                     & 33.71 & 49.81 & 49.96 \\
 & Ours, 4-bit              & 33.03 & 46.17 & 46.40 \\
 & \quad + early-site 8-bit & 33.39 & 49.81 & 50.11 \\
\midrule
\multirow{3}{*}{Qwen3-8B}
 & 16-bit                     & 38.51 & 82.49 & 82.56 \\
 & Ours, 4-bit              & 38.84 & 80.06 & 80.52 \\
 & \quad + early-site 8-bit & 39.95 & 81.73 & 82.18 \\
\bottomrule
\end{tabular}
\end{adjustbox}
\end{minipage}
\end{table}

\smallskip\noindent\textbf{Perplexity.}
Table~\ref{tab:ppl} reports WikiText perplexity, where lower is better. The
4-bit path raises perplexity by 0.29 on Llama-3-8B (6.07 to 6.36) and by 0.10
on Qwen3-8B (9.70 to 9.80), and the 8-bit path leaves it unchanged. An
increase of this size is small at these perplexity levels, and it is much smaller
than the increase that the raw quantizer error causes. Without compensation the 4-bit perplexity is
6.73 and 10.45, and the compensator closes most of the distance between
those values and the 16-bit model. Moving from 8 to 4 bits raises the
perplexity of Llama-3-8B by 0.29 with our all-reduce and by 0.56 with
FlashComm~\citep{flashcomm}. Without compensation our 4-bit path rises by 0.66,
close to the 0.56 of FlashComm, so the difference in perplexity comes from the
compensator. The task benchmarks below test whether the remaining
difference affects answers.

\begin{table*}[t]
\centering
\caption{\textbf{Downstream accuracy} (\%). The FlashComm rows are
quoted from~\citep{flashcomm,flashcommv2}, set in gray, and are read against
the 16-bit references of those papers. All other rows are measured in our
harness at $\TP=4$ and are read against the 16-bit row of the same model. Dashes mark results not reported in the cited work.}
\label{tab:downstream}
\footnotesize
\setlength{\tabcolsep}{5.5pt}
\renewcommand{\arraystretch}{1.15}
\begin{tabular}{@{}llccccccc@{}}
\toprule
Model & Method & PIQA & ARC-e & ARC-c & HellaSwag & WinoGrande & MMLU & MMLU-Pro \\
\midrule
\multirow{6}{*}{Llama-3-8B}
 & 16-bit                    & 80.52 & 77.69 & 53.39 & 79.26 & 73.80 & 61.54 & 32.34 \\
 & 4-bit, no compensation  & 78.18 & 72.22 & 48.38 & 77.35 & 70.64 & 58.82 & 28.76 \\
 & Ours, 8-bit             & 80.68 & 77.77 & 53.07 & 79.38 & 74.03 & 61.50 & 31.27 \\
 & Ours, 4-bit             & 79.65 & 76.68 & 51.79 & 78.93 & 70.72 & 59.01 & 30.13 \\
\cmidrule(l){2-9}
 & \textcolor{gray}{FlashComm 8-bit (quoted)} & \textcolor{gray}{80.58} & \textcolor{gray}{77.40} & \textcolor{gray}{52.47} & \textcolor{gray}{79.09} & \textcolor{gray}{73.09} & \textcolor{gray}{--} & \textcolor{gray}{--} \\
 & \textcolor{gray}{FlashComm 4-bit (quoted)} & \textcolor{gray}{80.09} & \textcolor{gray}{75.84} & \textcolor{gray}{51.02} & \textcolor{gray}{78.11} & \textcolor{gray}{70.48} & \textcolor{gray}{--} & \textcolor{gray}{--} \\
\midrule
\multirow{6}{*}{Qwen3-8B}
 & 16-bit                    & 78.28 & 80.22 & 56.81 & 78.76 & 72.22 & 73.02 & 47.20 \\
 & 4-bit, no compensation  & 76.77 & 77.73 & 55.54 & 76.47 & 67.10 & 71.08 & 45.51 \\
 & Ours, 8-bit             & 79.39 & 79.88 & 56.14 & 78.65 & 71.74 & 72.95 & 47.75 \\
 & Ours, 4-bit             & 79.38 & 81.21 & 56.00 & 77.81 & 70.32 & 72.65 & 46.38 \\
\cmidrule(l){2-9}
 & \textcolor{gray}{FlashComm 8-bit (quoted)} & \textcolor{gray}{77.58} & \textcolor{gray}{80.89} & \textcolor{gray}{56.57} & \textcolor{gray}{74.92} & \textcolor{gray}{68.11} & \textcolor{gray}{--} & \textcolor{gray}{--} \\
 & \textcolor{gray}{FlashComm 4-bit (quoted)} & \textcolor{gray}{76.82} & \textcolor{gray}{79.38} & \textcolor{gray}{55.63} & \textcolor{gray}{74.16} & \textcolor{gray}{67.72} & \textcolor{gray}{--} & \textcolor{gray}{--} \\
\bottomrule
\end{tabular}
\end{table*}

\smallskip\noindent\textbf{Downstream accuracy.}
Table~\ref{tab:downstream} reports seven standard benchmarks. At 8 bits, our
all-reduce matches 16-bit inference, with average scores within 0.12 points on
both models. At 4 bits, Qwen3-8B loses only 0.39 points on average and scores
higher than 16-bit on PIQA and ARC-easy, and Llama-3-8B loses 1.66 points on
average. The compensator keeps these losses small. Without it, the average
drop is 3.46 points on Llama-3-8B and 2.33 on Qwen3-8B, so compensation removes
about half of the loss on Llama-3-8B and 83\% of it on Qwen3-8B. Against the
16-bit row, the uncompensated 4-bit path loses significant accuracy on 10 of
the 14 model and task pairs. Compensation raises accuracy on all 14 pairs, and
significantly on 6 of them. On the five
tasks that FlashComm also reports, moving from 8 to 4 bits costs our all-reduce
1.43 points on Llama-3-8B and 0.22 on Qwen3-8B on average, against 1.42 and
0.87 for FlashComm.

\smallskip\noindent\textbf{Math reasoning.}
Table~\ref{tab:reasoning} reports GPQA and GSM8K. For GSM8K we report both
scoring modes of the harness: strict-match, which accepts the answer only in
the exact expected format, and flexible-extract, which accepts any correctly
extracted number. On GPQA the two paths differ by less than a point on both
models and in either direction, 33.03 against 33.71 on Llama-3-8B and 38.84
against 38.51 on Qwen3-8B. On GSM8K the 4-bit path scores 46.17
strict-match against 49.81 for 16-bit, a loss of 48 of 1,319 problems. The two
tasks differ in what they ask of the model. GPQA is answered with a single
choice, while a GSM8K solution is a generated chain of several hundred tokens
in which each token conditions on the ones before it, so one flipped
intermediate token can make the rest of the derivation wrong. The compensator keeps the
accumulated deviation small, but at 4 bits even a single rounding error
occasionally flips a token, and a long chain offers many chances. We
therefore spend more bits on a small part of the traffic. An error injected
at an early call enters the residual stream and reaches every later layer,
while an error injected near the output reaches only a few. In a long generated
chain, early errors therefore have more chances to change a token. We promote
the all-reduce calls of the first half of the layers, on each prefill and each
decode forward, to 8 bits, and call this policy early-site 8-bit. This recovers
most of the loss. It moves about 1.5 times the bytes of uniform 4-bit
communication. On
Llama-3-8B it recovers strict-match
to 49.81, the 16-bit score, and reaches 50.11 on flexible-extract against 49.96
for 16-bit. On Qwen3-8B it reaches 81.73 and 82.18 against 82.49 and 82.56 for
16-bit.
Early-site 8-bit is used only in the serving experiments of this section. The
training experiments of Sec.~\ref{subsec:exp-train} use uniform 4-bit
communication throughout.

\subsection{Training}\label{subsec:exp-train}

\begin{figure}[t]
\centering
\begin{minipage}[c]{0.485\linewidth}\centering
\includegraphics[width=\linewidth,trim=0 125 0 0,clip]{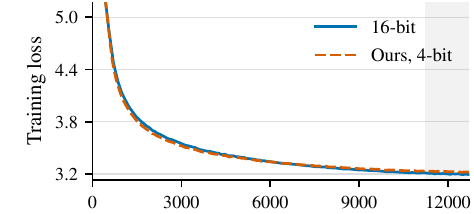}
\end{minipage}\hfill
\begin{minipage}[c]{0.485\linewidth}\centering
\includegraphics[width=\linewidth,trim=0 0 0 104.5,clip]{figures/train_c4_panels}
\end{minipage}
\caption{\textbf{Training loss} of Qwen3-0.6B at $\TP=4$ on 10B tokens
of Dolma C4, 16-bit at peak LR $3.0\times10^{-4}$ against our 4-bit all-reduce
at peak LR $3.5\times10^{-4}$. Left:
full run (50-step moving average, y axis clipped, gray band marks the steps on the right). Right: last 1,500 steps
(100-step average, $\pm$1 s.d. band).}\label{fig:train-c4}
\end{figure}

\begin{table}[t]
\centering
\caption{\textbf{Training loss} of Qwen3-0.6B at $\TP=4$ on Dolma C4.
(a) Full 10B-token training runs. (b) A separate 4-bit learning-rate
ablation with batch size 32 and 4,096 steps. Tail-100 is the mean loss over
the last 100 steps. In (a), $\Delta$ is relative to the 16-bit run, and in (b),
$\Delta_{3.5}$ is relative to the run at LR $3.5\times10^{-4}$.
All differences use unrounded tail-100 losses.}\label{tab:train}
\footnotesize
\setlength{\tabcolsep}{3pt}
\renewcommand{\arraystretch}{1.15}
\begin{minipage}[t]{0.485\linewidth}\centering
{\footnotesize\emph{(a) Full training: 10B tokens, batch size 384}}\par\vspace{3pt}
\begin{tabular*}{\linewidth}{@{\extracolsep{\fill}}lcccc@{}}
\toprule
Method & \shortstack{Max LR\\($\times10^{-4}$)}
 & Step 1000 & Tail-100 & $\Delta$ \\
\midrule
16-bit       & 3.0 & 4.095 & 3.197 & -- \\
Ours, 4-bit  & 3.5 & 4.042 & 3.225 & +0.027 \\
\bottomrule
\end{tabular*}
\end{minipage}\hfill
\begin{minipage}[t]{0.485\linewidth}\centering
{\footnotesize\emph{(b) LR ablation: batch size 32, 4,096 steps}}\par\vspace{3pt}
\begin{tabular*}{\linewidth}{@{\extracolsep{\fill}}lcccc@{}}
\toprule
Max LR ($\times10^{-4}$) & 2.0 & 2.5 & 3.0 & 3.5 \\
\midrule
Tail-100 loss & 4.135 & 4.078 & 4.016 & 3.952 \\
$\Delta_{3.5}$ & +0.183 & +0.126 & +0.064 & 0.000 \\
\bottomrule
\end{tabular*}
\end{minipage}
\end{table}

We pretrain Qwen3-0.6B~\citep{qwen3} from scratch at $\TP=4$ with the 4-bit
all-reduce at every forward and backward site. The main experiment trains on
10B tokens of Dolma C4~\citep{soldaini2024dolma,raffel2020t5} with a global
batch of 384 sequences for 12,716 steps, using AdamW, 100 warmup steps, and
cosine decay (Appendix~\ref{app:train}). The 16-bit and 4-bit runs share the
data, batch size, optimizer, and schedule. Low-bit training has been observed to
benefit from a larger learning rate than its full-precision
counterpart~\citep{wang2023bitnet,loco}, so we give the 4-bit run a slightly larger
peak learning rate, $3.5\times10^{-4}$ against $3.0\times10^{-4}$
for 16-bit.

Figure~\ref{fig:train-c4} compares the two runs. Their loss curves track each
other closely over all 12,716 steps, and the right panel magnifies the last
1,500 steps. Over the last 100 steps (tail-100), the 4-bit loss averages 3.225
and the 16-bit loss 3.197, a gap of 0.027 computed before rounding, less than
0.9\% of the 16-bit loss. Early in training the
4-bit run, with its larger learning rate, has the lower loss, 4.042 against 4.095 at step
1,000 (Table~\ref{tab:train}a).

Table~\ref{tab:train}b supports this choice of learning rate. In separate,
shorter 4-bit runs on Dolma C4 with a global batch of 32 sequences for 4,096
steps, the tail-100 loss falls steadily as the peak learning rate rises, from
4.135 at $2.0\times10^{-4}$ to 4.078, 4.016, and 3.952 at $2.5$, $3.0$, and
$3.5\times10^{-4}$. The final loss is sensitive to the learning
rate~\citep{xie2024optimization}, and the 4-bit run gains from the larger
value.

\subsection{Ablations}\label{subsec:exp-abl}

\begin{table}[t]
\centering
\caption{\textbf{Component ablations}, Llama-3-8B at $\TP=4$, WikiText and C4
perplexity (WT = WikiText), lower is better. Each row changes one component
of the method, and the $\gsz=256$ rows also change the group size. Bold marks the minimum within each bit-width column, excluding
the 16-bit reference.}\label{tab:abl}
\footnotesize
\setlength{\tabcolsep}{2.5pt}
\renewcommand{\arraystretch}{1.15}
\begin{minipage}[t]{0.475\linewidth}\centering\vspace{0pt}
\begin{tabular*}{\linewidth}{@{\extracolsep{\fill}}lcccc@{}}
\toprule
 & \multicolumn{2}{c}{4-bit} & \multicolumn{2}{c}{8-bit} \\
\cmidrule(lr){2-3}\cmidrule(l){4-5}
 & WT $\downarrow$ & C4 $\downarrow$ & WT $\downarrow$ & C4 $\downarrow$ \\
\midrule
\textbf{Full method} & \textbf{6.362} & 9.484 & \textbf{6.070} & \textbf{8.970} \\
\addlinespace[3pt]
\multicolumn{5}{@{}l}{\emph{Quantizer}} \\
\multicolumn{5}{@{}l}{\quad\textcolor{black!60}{(full: zero-aware asymmetric, $\gsz=128$)}} \\
\quad naive zero-point & 7.296 & 11.260 & 6.072 & 8.973 \\
\quad symmetric grid & 6.568 & 9.826 & 6.075 & 8.979 \\
\quad symmetric grid, $\gsz=256$ & 6.921 & 10.611 & 6.081 & 8.990 \\
\quad naive zero-point, $\gsz=256$ & 9.204 & 14.582 & 6.075 & 8.978 \\
\addlinespace[3pt]
\multicolumn{5}{@{}l}{\emph{Compensator}} \\
\multicolumn{5}{@{}l}{\quad\textcolor{black!60}{(full: dual, two averages)}} \\
\quad single average ($\beta=0.4$) & 6.506 & 9.796 & 6.071 & 8.972 \\
\quad no compensation & 6.730 & 10.130 & 6.079 & 8.984 \\
\bottomrule
\end{tabular*}
\end{minipage}\hfill
\begin{minipage}[t]{0.475\linewidth}\centering\vspace{0pt}
\begin{tabular*}{\linewidth}{@{\extracolsep{\fill}}lcccc@{}}
\toprule
 & \multicolumn{2}{c}{4-bit} & \multicolumn{2}{c}{8-bit} \\
\cmidrule(lr){2-3}\cmidrule(l){4-5}
 & WT $\downarrow$ & C4 $\downarrow$ & WT $\downarrow$ & C4 $\downarrow$ \\
\midrule
\multicolumn{5}{@{}l}{\emph{Coefficients}} \\
\multicolumn{5}{@{}l}{\quad\textcolor{black!60}{(full: $\eta=0.75$, $a=0.35$, $(\beta_1,\beta_2)=(0.7,0.3)$)}} \\
\quad $\eta=0.7$ & 6.365 & 9.469 & 6.071 & 8.972 \\
\quad $\eta=0.9$ & 6.384 & 9.525 & 6.071 & 8.972 \\
\quad $a=0.25$ & 6.365 & \textbf{9.467} & 6.072 & 8.972 \\
\quad $a=0.45$ & 6.367 & 9.481 & 6.072 & 8.971 \\
\quad $(\beta_1,\beta_2)=(0.8,0.2)$ & 6.370 & 9.489 & 6.071 & 8.973 \\
\quad $(\beta_1,\beta_2)=(0.6,0.4)$ & 6.370 & 9.489 & 6.071 & 8.973 \\
\midrule
16-bit reference & 6.070 & 8.966 & 6.070 & 8.966 \\
\bottomrule
\end{tabular*}
\end{minipage}
\end{table}

We now remove or replace one component of the method at a time and measure
the effect on perplexity, where lower is better. All runs serve Llama-3-8B at
$\TP=4$ with the same evaluation as Table~\ref{tab:ppl}, on WikiText and C4.

Table~\ref{tab:abl} evaluates every variant at 4 and at 8 bits. At 8 bits the
variants span only 0.011 perplexity on WikiText and 0.020 on C4. At 4 bits they
span 2.842 and 5.115. At 4 bits, the full design has much lower perplexity than
the naive zero-point and no-compensation variants. Changing any single
coefficient moves the 4-bit perplexity by at most 0.022 on WikiText and 0.041 on
C4, so the method is not sensitive to these coefficients. We examine the 4-bit
columns below.

The first factor is the group size. Coarsening the groups from 128 to 256
raises perplexity by 0.35 on WikiText and 0.79 on C4 with the symmetric
quantizer, and by 1.9 and 3.3 with the naive zero-point. A group has sixteen
levels and one scale, so groups must stay small for the scale to resolve the
values it covers. Since each group also carries its own scale and zero-point,
finer groups add bytes to the message (Sec.~\ref{subsec:exp-speed}), and we
use $\gsz=128$.

The second factor is the zero-point. With a naive zero-point the asymmetric
quantizer loses to the symmetric grid by 0.73 on WikiText and 1.43 on C4. As
shown in Sec.~\ref{subsec:quant}, the naive map
represents zero inexactly, and the error this injects into the densest part
of the activation distribution outweighs the gain from the finer resolution
of the fitted range. The zero-aware assignment removes this error, and the asymmetric
quantizer reaches 6.362 on WikiText and 9.484 on C4, lower than both
alternatives.

The third factor is the compensator. Removing it makes each call quantize
independently, so the errors accumulate across sites, which is the
setting of one-shot compression in prior work. Perplexity rises to 6.730 on
WikiText and 10.130 on C4.
The accuracy counterpart of this row appears in Table~\ref{tab:downstream}.
Keeping the compensator but collapsing its two averages into one loses
0.14 on WikiText and 0.31 on C4, the contribution of the second time scale. At
8 bits the same substitution moves C4 perplexity only from 8.970 to 8.972.

The last rows vary the coefficients. The decay rates and the blend enter the
guarantee only through a constant (Sec.~\ref{subsec:weights}). Consistent
with this, every $a$ and $(\beta_1,\beta_2)$ row stays within 0.022 of
the full method. The gain
trades two errors against each other. At $\eta=1$ nothing is left uncorrected
(Eqn.~\eqref{eq:gain-bound}), but the correction is an estimate of an error
the compensator has not yet seen, and at full strength an imperfect estimate
overcorrects whenever the error changes between sites. A gain below one leaves
the fraction $1-\eta$ of each error uncorrected in exchange for less
overcorrection. The theory is
stated at $\eta=1$, and the sweep favors a smaller gain. The gain $\eta=0.9$ is worse than
both 0.7 and 0.75 on WikiText and C4. We use $\eta=0.75$, $a=0.35$, $\beta_1=0.7$, and
$\beta_2=0.3$ throughout, and nothing is tuned per model.

The 4-bit accuracy of Sec.~\ref{subsec:exp-acc} thus depends on both
components. The quantizer bounds the error each site injects, and the
compensator keeps the remaining error from accumulating.

\section{Conclusion}\label{sec:conclusion}
\method{} starts from the arithmetic constraint of a low-bit collective. The
all-reduce must add the values it moves, and the sum cannot be formed in the
transport format. The collective is therefore built as low-bit transport around
a high-precision reduction, which quantizes each element at most twice per call
for any tensor-parallel degree. The
reduction also reveals each injected error at no extra communication, and the
fused compensator feeds it back into later calls, so the accumulated error
stays bounded with guarantees independent of network depth. At 8 bits, inference matches 16-bit accuracy, and 4-bit inference stays
within 0.4 and 1.7 points of it on average over seven tasks for Qwen3-8B and
Llama-3-8B. With a slightly larger learning rate, 4-bit pretraining over 10B tokens
ends less than 0.9\% above the 16-bit loss.
For 1\,GB messages the compensated all-reduce is 2.09$\times$ faster than the
16-bit NCCL all-reduce.

The end-to-end gains grow with the traffic on the interconnect. The time to first token
of Llama-3-8B improves by 1.16$\times$ without background load and by
1.44$\times$ at 100\% load, because communication then takes a larger share of each step. The compensation itself relies
on two structural properties. The collective reveals the error it injects, and
its outputs join the same residual stream call after call, so a later
correction can cancel an earlier error. The reduce-scatter of sequence
parallelism and the combine step of expert parallelism share both, and are natural next targets
for the compensated collective.

\setlength{\emergencystretch}{1.5em}
\bibliographystyle{abbrvnat}
\bibliography{references,extra}

\clearpage
\appendix

\section{Complete Theoretical Proofs}\label{app:proofs}

This appendix proves the results of Sec.~\ref{sec:theory} in the order in which
the section states them: notation and regularity, Theorem~\ref{thm:source} together
with its gain and variance forms, the partial-sum and coefficient forms behind
Eqn.~\eqref{eq:meansq}, the completion of Assumption~\ref{asm:noreinf},
Theorem~\ref{thm:prop}, the constructions behind the sharpness
discussion of Sec.~\ref{subsec:bounds}, and the backward source bound.

\subsection{Notation and regularity}

Sites are indexed $k=0,\dots,\Nsite-1$. The clean trajectory $\mathbf{z}_k$, the
deviation $\mathbf{E}_k$, the Jacobians $\mathbf{J}_k$, and the products
$\bm{\Phi}_{n,j}=(\mathbf{I}+\mathbf{J}_{n-1})\cdots(\mathbf{I}+\mathbf{J}_j)$,
with $\bm{\Phi}_{n,n}=\mathbf{I}$, are those of Eqn.~\eqref{eq:carry}. The scale is
$\sigma_k^2=\norm{\mathbf{z}_k}^2/d$, with increment
$\Delta_k=\sigma_{k+1}^2-\sigma_k^2$. Three per-block quantities enter the
constants, all uniform over the class of trained checkpoints considered, as is the
budget $\Lambda$ of Assumption~\ref{asm:noreinf}:
\begin{align*}
  \text{(R1)}\quad & 0<\Delta_{\min}\le\Delta_k\le\Delta_{\max}, \\
  \text{(R2)}\quad & \norm{\mathbf{J}_k}_F^2/d\le C_0\,\Delta_k/\sigma_k^2, \\
  \text{(R3)}\quad & H_k\le\bar h/(\sigma_k^2\sqrt d),
\end{align*}
with $H_k$ the second-derivative bound of Eqn.~\eqref{eq:err-expand}. R1 records
that every site adds a bounded, nonvanishing amount of energy to the hidden state,
and R2 and R3 record the scaling that the normalization in front of each block
imposes on its derivatives. A bound is called depth independent when its constant
is built from $\Delta_{\min},\Delta_{\max},C_0,\bar h$, the scales
$\sigma_0,\sigma_1$, the compensator constants $\mu_w,\tau_2$, the budget
$\Lambda$ of Assumption~\ref{asm:noreinf}, and possibly the hidden dimension $d$,
and from nothing that grows with $\Nsite$.

With $\mathcal{F}_s$ the history of the injections up to and including site $s$,
the noise model is
\begin{equation}\label{eq:noise-model}
  \mathbb{E}[\injnoise_s\mid\mathcal{F}_{s-1}]=\mathbf{0},
  \qquad
  \mathbb{E}[\injnoise_s\injnoise_s^\top\mid\mathcal{F}_{s-1}]\preceq\frac{\qbound^2}{d}\,\mathbf{I},
  \qquad
  \norm{\injnoise_s}\le\qbound\ \text{a.s.}
\end{equation}
The clean trajectory is fixed, so the $\mathbf{J}_k$ and $\bm{\Phi}_{n,j}$ are
deterministic and all probability is over the injected errors.

\subsection{The source bound}

\begin{lemma}[tail representation]\label{lem:tails}
At gain $\eta=1$, $\;\mathbf{S}_t=\sum_{s\le t}T_{t-s}\,\injnoise_s$.
\end{lemma}

\begin{proof}
Equation~\eqref{eq:coeff} gives the coefficients $\phi_m=1-\eta\sum_{r<m}w_r$. At
$\eta=1$, since $\sum_rw_r=1$,
\[
  \phi_m=1-\sum_{r<m}w_r=\sum_{r\ge m}w_r=T_m .\qedhere
\]
\end{proof}

\begin{proof}[Proof of Theorem~\ref{thm:source}]
The $T_m$ are nonnegative, so for every sequence with
$\norm{\injnoise_k}\le\qbound$,
\[
  \norm{\mathbf{S}_t}
  \;\le\;\sum_{s\le t}T_{t-s}\norm{\injnoise_s}
  \;\le\;\qbound\sum_{m\ge0}T_m
  \;=\;\qbound\sum_{m\ge0}\sum_{r\ge m}w_r
  \;=\;\qbound\sum_{r\ge0}(r+1)\,w_r
  \;=\;(1+\mu_w)\,\qbound .
\]
Without compensation $\phi_m\equiv1$, and for any fixed unit vector $\mathbf{e}$
the constant sequence $\injnoise_k=\qbound\,\mathbf{e}$ attains
\[
  \norm{\mathbf{S}_{\Nsite-1}}=\Bigl\|\sum_{k<\Nsite}\qbound\,\mathbf{e}\Bigr\|=\Nsite\qbound .
\]
At gain $\eta<1$ the coefficient is $\phi_m=(1-\eta)+\eta T_m$, and the same
estimate gives
\[
  \max_t\norm{\mathbf{S}_t}\;\le\;(1-\eta)\,\Nsite\qbound+\eta\,(1+\mu_w)\,\qbound,
\]
which is Eqn.~\eqref{eq:gain-bound}. As $\beta_1,\beta_2\to0$ the kernel
concentrates at age zero, $T_0=1$, $T_{m\ge1}=0$, and the bound is $\qbound$.
\end{proof}

\begin{remark}[source variance]\label{rem:srcvar}
Under the noise model the cross terms vanish, since for $s<s'$,
$\mathbb{E}\langle\injnoise_s,\injnoise_{s'}\rangle
=\mathbb{E}\langle\injnoise_s,\mathbb{E}[\injnoise_{s'}\mid\mathcal{F}_{s'-1}]\rangle=0$.
Hence, by Lemma~\ref{lem:tails},
\[
  \mathbb{E}\,\norm{\mathbf{S}_t}^2
  =\sum_{s\le t}T_{t-s}^2\,\mathbb{E}\norm{\injnoise_s}^2
  \;\le\;\qbound^2\sum_{m\ge0}T_m^2\;=\;\tau_2\,\qbound^2 ,
\]
the variance form quoted after Theorem~\ref{thm:prop}. At gain $\eta$ the same
computation with $\phi_m=(1-\eta)+\eta T_m$ gives
$\mathbb{E}\norm{\mathbf{S}_t}^2\le\qbound^2\sum_{m\le t}\phi_m^2$.
\end{remark}

\subsection{The linearized deviation}

The linearized deviation is
$\mathbf{E}^{\mathrm{lin}}_\Nsite=\sum_t\bm{\Phi}_{\Nsite,t+1}\bm{\rho}_t$, the
first-order part of Eqn.~\eqref{eq:carry}. From here on the gain is $\eta=1$.

\begin{lemma}[partial-sum form]\label{lem:psum}
\[
  \mathbf{E}^{\mathrm{lin}}_\Nsite
  =\mathbf{S}_{\Nsite-1}
  +\sum_{t=0}^{\Nsite-2}\bm{\Phi}_{\Nsite,t+2}\,\mathbf{J}_{t+1}\,\mathbf{S}_t .
\]
\end{lemma}

\begin{proof}
Since $\bm{\rho}_t=\mathbf{S}_t-\mathbf{S}_{t-1}$ with
$\mathbf{S}_{-1}=\mathbf{0}$ (Eqn.~\eqref{eq:partial-sum}), summation by parts
with
$\bm{\Phi}_{\Nsite,t+1}-\bm{\Phi}_{\Nsite,t+2}=\bm{\Phi}_{\Nsite,t+2}\mathbf{J}_{t+1}$
and $\bm{\Phi}_{\Nsite,\Nsite}=\mathbf{I}$ gives
\[
  \sum_{t=0}^{\Nsite-1}\bm{\Phi}_{\Nsite,t+1}(\mathbf{S}_t-\mathbf{S}_{t-1})
  =\bm{\Phi}_{\Nsite,\Nsite}\,\mathbf{S}_{\Nsite-1}
  +\sum_{t=0}^{\Nsite-2}\bigl(\bm{\Phi}_{\Nsite,t+1}-\bm{\Phi}_{\Nsite,t+2}\bigr)\mathbf{S}_t .
  \qedhere
\]
\end{proof}

The deviation therefore depends on the injected errors only through their partial
sums, the form promised in Sec.~\ref{subsec:decomposition}. The final partial
sum is carried by the identity path to the output, and each earlier partial sum enters
through the Jacobian of the site that follows it and is then propagated forward.
This is where the compensator enters the output bound, through the partial sums
$\mathbf{S}_t$.

\begin{lemma}[coefficient matrices]\label{lem:ksmat}
\[
  \mathbf{E}^{\mathrm{lin}}_\Nsite=\sum_{s=0}^{\Nsite-1}\mathbf{K}_s\,\injnoise_s,
  \qquad
  \mathbf{K}_s=T_{\Nsite-1-s}\,\mathbf{I}
  +\sum_{t=s}^{\Nsite-2}T_{t-s}\,\bm{\Phi}_{\Nsite,t+2}\,\mathbf{J}_{t+1}.
\]
\end{lemma}

\begin{proof}
Substitute $\mathbf{S}_t=\sum_{s\le t}T_{t-s}\injnoise_s$ (Lemma~\ref{lem:tails})
into Lemma~\ref{lem:psum} and collect the coefficient of $\injnoise_s$: the term
$\mathbf{S}_{\Nsite-1}$ contributes $T_{\Nsite-1-s}\,\mathbf{I}$, and the term at
$t\ge s$ contributes $T_{t-s}\,\bm{\Phi}_{\Nsite,t+2}\mathbf{J}_{t+1}$.
\end{proof}

The first term of $\mathbf{K}_s$ is the tail of the error not yet returned by the
end of the pass, carried by the identity path. The second collects the corrections
applied at the later sites, each entering through the block Jacobian there and
then propagated to the output.

\begin{lemma}[mean square]\label{lem:msid}
Under Eqn.~\eqref{eq:noise-model}, Eqn.~\eqref{eq:meansq} holds:
$\;\mathbb{E}\,\|\mathbf{E}^{\mathrm{lin}}_\Nsite\|^2\le(\qbound^2/d)\sum_s\norm{\mathbf{K}_s}_F^2$.
\end{lemma}

\begin{proof}
The $\mathbf{K}_s$ are deterministic, so conditioning on $\mathcal{F}_{t-1}$,
\[
  \mathbb{E}\langle\mathbf{K}_s\injnoise_s,\mathbf{K}_t\injnoise_t\rangle
  =\mathbb{E}\bigl\langle\mathbf{K}_s\injnoise_s,\ \mathbf{K}_t\,
  \mathbb{E}[\injnoise_t\mid\mathcal{F}_{t-1}]\bigr\rangle=0,
  \qquad s<t .
\]
Hence
\[
  \mathbb{E}\,\|\mathbf{E}^{\mathrm{lin}}_\Nsite\|^2
  =\sum_s\mathbb{E}\norm{\mathbf{K}_s\injnoise_s}^2
  =\sum_s\mathbb{E}\,\operatorname{tr}\bigl(\mathbf{K}_s^\top\mathbf{K}_s\,\injnoise_s\injnoise_s^\top\bigr)
  \;\le\;\frac{\qbound^2}{d}\sum_s\norm{\mathbf{K}_s}_F^2 .\qedhere
\]
\end{proof}

\subsection{The completed assumption}

The statement in Sec.~\ref{subsec:bounds} covers the deviations
$\bm{\Phi}_{k,s+1}\injnoise_s$, with covariances proportional to
$(\qbound^2/d)\,\bm{\Phi}_{k,s+1}\bm{\Phi}_{k,s+1}^\top$. The proofs apply the same inequality to
two further collections produced by the same propagation. The completed family
$\mathcal{C}_k$ is the closed convex hull of the normalizations of
\begin{equation}\label{eq:family}
  \bm{\Phi}_{k,\ell}\bm{\Phi}_{k,\ell}^\top,
  \qquad
  \bm{\Phi}_{k,t+2}\mathbf{J}_{t+1}\mathbf{J}_{t+1}^\top\bm{\Phi}_{k,t+2}^\top,
  \qquad
  \bm{\Phi}_{k,j+1}\,\mathbb{E}\bigl[\mathbf{r}^\tau_j(\mathbf{r}^\tau_j)^\top\bigr]\,\bm{\Phi}_{k,j+1}^\top,
\end{equation}
over $\ell\le k$, $t+2\le k$, and $j+1\le k$, with $\mathbf{r}^\tau_j$ the stopped
remainder of Proposition~\ref{prop:remainder} below, included at
the noise level $\qbound$ in use. Every member is the covariance of a deviation
that injected noise, or its second-order response, produces at site $k$. The
third collection varies with the truncation threshold of
Proposition~\ref{prop:remainder}, which depends on $p$ and $\Nsite$, and the
assumption is understood to hold with the same budget $\Lambda$ for every
truncation threshold. For $\bm{\Sigma}=\mathbb{E}\,\mathbf{v}\mathbf{v}^\top$
the inequality of Assumption~\ref{asm:noreinf} is
\begin{equation}\label{eq:asm-trace}
  \operatorname{tr}\bigl((\mathbf{I}+\mathbf{J}_k)^\top(\mathbf{I}+\mathbf{J}_k)\bm{\Sigma}\bigr)
  \;\le\;(1+\varepsilon_k)\,\frac{\sigma_{k+1}^2}{\sigma_k^2}\,\operatorname{tr}\bm{\Sigma},
  \qquad\bm{\Sigma}\in\mathcal{C}_k .
\end{equation}

\begin{lemma}[propagation of covariances]\label{lem:covprop}
Let $\bm{\Sigma}$ be a member of one of the three collections in
Eqn.~\eqref{eq:family} at site $j$. Then for $j\le n\le\Nsite$,
\[
  \frac{\operatorname{tr}\bigl(\bm{\Phi}_{n,j}\bm{\Sigma}\bm{\Phi}_{n,j}^\top\bigr)}{\sigma_n^2}
  \;\le\;e^{\Lambda}\,\frac{\operatorname{tr}\bm{\Sigma}}{\sigma_j^2}.
\]
\end{lemma}

\begin{proof}
Set $\bm{\Sigma}_m=\bm{\Phi}_{m,j}\bm{\Sigma}\bm{\Phi}_{m,j}^\top$, so that
$\bm{\Sigma}_{m+1}=(\mathbf{I}+\mathbf{J}_m)\bm{\Sigma}_m(\mathbf{I}+\mathbf{J}_m)^\top$.
Since $\bm{\Phi}_{m,j}\bm{\Phi}_{j,\ell}=\bm{\Phi}_{m,\ell}$, each
$\bm{\Sigma}_m$ again belongs to the same collection of Eqn.~\eqref{eq:family} at
site $m$, so Eqn.~\eqref{eq:asm-trace} applies at every step,
\[
  \frac{\operatorname{tr}\bm{\Sigma}_{m+1}}{\sigma_{m+1}^2}
  \;\le\;(1+\varepsilon_m)\,\frac{\operatorname{tr}\bm{\Sigma}_m}{\sigma_m^2},
\]
and $\prod_{m}(1+\varepsilon_m)\le e^{\sum_m\varepsilon_m}\le e^{\Lambda}$.
\end{proof}

\begin{lemma}[propagated sources]\label{lem:gamma2}
Under Assumption~\ref{asm:noreinf},
\[
  \sum_{t=0}^{\Nsite-2}\frac{\norm{\bm{\Phi}_{\Nsite,t+2}\mathbf{J}_{t+1}}_F^2}{d}
  \le e^{\Lambda}C_0\,\frac{\sigma_\Nsite^2}{\sigma_1^2}.
\]
\end{lemma}

\begin{proof}
The matrix $\mathbf{J}_{t+1}\mathbf{J}_{t+1}^\top$ is the member of the second
collection at site $j=t+2$, the case $k=t+2$, where
$\bm{\Phi}_{t+2,t+2}=\mathbf{I}$. Lemma~\ref{lem:covprop}, then R2 and
$\Delta_{t+1}/(\sigma_{t+1}^2\sigma_{t+2}^2)=1/\sigma_{t+1}^2-1/\sigma_{t+2}^2$,
give
\[
  \frac{\norm{\bm{\Phi}_{\Nsite,t+2}\mathbf{J}_{t+1}}_F^2}{\sigma_\Nsite^2}
  \;\le\;e^{\Lambda}\,\frac{\norm{\mathbf{J}_{t+1}}_F^2}{\sigma_{t+2}^2}
  \;\le\;e^{\Lambda}C_0\,d\Bigl(\frac{1}{\sigma_{t+1}^2}-\frac{1}{\sigma_{t+2}^2}\Bigr),
\]
and the sum over $t$ telescopes to at most $e^{\Lambda}C_0\,d/\sigma_1^2$.
\end{proof}

\subsection{Proof of the propagation theorem}

\begin{proof}[Proof of Theorem~\ref{thm:prop}, part (i)]
Write $f_t=\norm{\bm{\Phi}_{\Nsite,t+2}\mathbf{J}_{t+1}}_F$. Taking Frobenius norms
in Lemma~\ref{lem:ksmat},
\[
  \norm{\mathbf{K}_s}_F\;\le\;\sqrt d\,T_{\Nsite-1-s}+\sum_{t\ge s}T_{t-s}f_t .
\]
Squaring, summing over $s$, and applying Young's inequality
$\|T*f\|_{\ell_2}\le\|T\|_{\ell_1}\|f\|_{\ell_2}$ with $\|T\|_{\ell_1}=1+\mu_w$,
then Lemma~\ref{lem:gamma2},
\begin{align*}
  \sum_s\norm{\mathbf{K}_s}_F^2
  &\;\le\;2d\,\tau_2+2\,(1+\mu_w)^2\sum_tf_t^2
  \;\le\;2d\,\tau_2+2\,(1+\mu_w)^2\,d\,e^{\Lambda}C_0\,\frac{\sigma_\Nsite^2}{\sigma_1^2}.
\end{align*}
Inserting into Lemma~\ref{lem:msid} and dividing by $d\sigma_\Nsite^2$,
\[
  \frac{\mathbb{E}\,\|\mathbf{E}^{\mathrm{lin}}_\Nsite\|^2}{d\sigma_\Nsite^2}
  \;\le\;\frac{\qbound^2}{d^2\sigma_\Nsite^2}
  \Bigl[2d\,\tau_2+2\,(1+\mu_w)^2\,d\,e^{\Lambda}C_0\,\frac{\sigma_\Nsite^2}{\sigma_1^2}\Bigr]
  \;=\;2\Bigl[\tau_2\frac{\sigma_1^2}{\sigma_\Nsite^2}+(1+\mu_w)^2C_0e^{\Lambda}\Bigr]\bar{q}^2 .
  \qedhere
\]
\end{proof}

\begin{remark}[without compensation]\label{rem:nocomp}
Here $\bm{\rho}_s=\injnoise_s$ and
$\mathbf{K}_s=\bm{\Phi}_{\Nsite,s+1}$. The identity is the member of the first
collection of Eqn.~\eqref{eq:family} at site $s+1$ (the case $\ell=k$), so
Lemma~\ref{lem:covprop} gives
$\norm{\bm{\Phi}_{\Nsite,s+1}}_F^2\le e^{\Lambda}d\,\sigma_\Nsite^2/\sigma_{s+1}^2$,
and Lemma~\ref{lem:msid} with
$\sigma_{s+1}^2\ge\sigma_1^2+s\Delta_{\min}$ (R1) gives
\[
  \frac{\mathbb{E}\,\|\mathbf{E}^{\mathrm{lin}}_\Nsite\|^2}{d\sigma_\Nsite^2}
  \;\le\;e^{\Lambda}\bar{q}^2\sum_{s=0}^{\Nsite-1}\frac{\sigma_1^2}{\sigma_{s+1}^2}
  \;\le\;e^{\Lambda}\bar{q}^2\Bigl[1+\frac{\sigma_1^2}{\Delta_{\min}}
  \log\Bigl(1+\frac{(\Nsite-1)\Delta_{\min}}{\sigma_1^2}\Bigr)\Bigr].
\]
Each site enters with the weight $1/\sigma_{s+1}^2$, and the sum grows as
$\log\Nsite$. With compensation the same weight enters only through
Lemma~\ref{lem:gamma2}, multiplied by
$C_0\,\Delta_{t+1}/\sigma_{t+2}^2\le C_0\,\Delta_{t+1}/\sigma_{t+1}^2$, and the sum
telescopes to $e^{\Lambda}C_0\,d/\sigma_1^2$, which gives the term
$(1+\mu_w)^2C_0e^{\Lambda}$ of part (i).
\end{remark}

\begin{corollary}[general gain]\label{cor:gain}
Under the hypotheses of Theorem~\ref{thm:prop}, at gain $\eta\in[0,1]$,
\[
  \frac{\mathbb{E}\,\|\mathbf{E}^{\mathrm{lin}}_\Nsite\|^2}{d\sigma_\Nsite^2}
  \;\le\;4\eta^2\Big[\tau_2\frac{\sigma_1^2}{\sigma_\Nsite^2}+(1+\mu_w)^2C_0e^{\Lambda}\Big]\bar{q}^2
  +2(1-\eta)^2e^{\Lambda}\bar{q}^2\sum_{s=0}^{\Nsite-1}\frac{\sigma_1^2}{\sigma_{s+1}^2}.
\]
\end{corollary}

\begin{proof}
By Eqn.~\eqref{eq:kernel} the correction is linear in the gain,
$\mathbf{C}_k=\eta\,\mathbf{C}^{(1)}_k$ with $\mathbf{C}^{(1)}_k$ the correction
at $\eta=1$, so $\bm{\rho}_k=(1-\eta)\,\injnoise_k+\eta\,(\injnoise_k+\mathbf{C}^{(1)}_k)$.
The linearized deviation is linear in the net errors, so it is $1-\eta$ times the
uncompensated deviation plus $\eta$ times the deviation at $\eta=1$, both driven by
the same injections. The bound follows from $(x+y)^2\le2x^2+2y^2$, part (i), and
Remark~\ref{rem:nocomp}.
\end{proof}

\begin{proof}[Proof of part (ii)]
The sums $\mathbf{M}_n=\sum_{s\le n}\mathbf{K}_s\injnoise_s$ form a martingale in
$\mathbb{R}^d$. Its predictable quadratic variation is at most the deterministic
majorant
\[
  v_\Nsite:=\frac{\qbound^2}{d}\sum_s\norm{\mathbf{K}_s}_F^2,
  \qquad
  \norm{\mathbf{K}_s\injnoise_s}\le\zeta_\Nsite:=\qbound\max_s\norm{\mathbf{K}_s} .
\]
By the first display in the proof of Lemma~\ref{lem:gamma2}, with
$\Delta_{t+1}\le\Delta_{\max}$,
\[
  \norm{\bm{\Phi}_{\Nsite,t+2}\mathbf{J}_{t+1}}
  \;\le\;\sigma_\Nsite\,\frac{\sqrt{d\,e^{\Lambda}C_0\,\Delta_{\max}}}{\sigma_{t+1}\sigma_{t+2}},
\]
so, by Lemma~\ref{lem:ksmat}, $T_m\le1$, $\sum_mT_m=1+\mu_w$, and
$\sigma_{t+1}\sigma_{t+2}\ge\sigma_1^2$,
\[
  \norm{\mathbf{K}_s}
  \;\le\;1+(1+\mu_w)\,\frac{\sigma_\Nsite}{\sigma_1^2}\,\sqrt{d\,e^{\Lambda}C_0\,\Delta_{\max}} .
\]
The Bennett--Bernstein inequality for Hilbert-space
martingales~\citep{pinelis1994}, applied with $v_\Nsite$ and $\zeta_\Nsite$, gives
\[
  \Pr\Bigl[\max_n\norm{\mathbf{M}_n}\ge\sqrt{2v_\Nsite \lambda_p}+\tfrac23\zeta_\Nsite \lambda_p\Bigr]\le p,
  \qquad \lambda_p=\log\tfrac2p .
\]
On the complementary event, $(x+y)^2\le2x^2+2y^2$ gives
$\|\mathbf{E}^{\mathrm{lin}}_\Nsite\|^2\le4v_\Nsite \lambda_p+\tfrac89\zeta_\Nsite^2\lambda_p^2$.
Dividing by $d\sigma_\Nsite^2$ and using part (i) for $v_\Nsite$ and the display
above for $\zeta_\Nsite$,
\[
  \frac{\|\mathbf{E}^{\mathrm{lin}}_\Nsite\|^2}{d\sigma_\Nsite^2}
  \;\le\;C\,\bar{q}^2\lambda_p+C'\bar{q}^2\lambda_p^2,
  \qquad
  C:=8\Bigl[\tau_2\frac{\sigma_1^2}{\sigma_\Nsite^2}+(1+\mu_w)^2C_0e^{\Lambda}\Bigr],
\]
\[
  C':=\frac89\Bigl(\frac{\sigma_1}{\sigma_\Nsite}
  +(1+\mu_w)\,\frac{\sqrt{d\,e^{\Lambda}C_0\,\Delta_{\max}}}{\sigma_1}\Bigr)^2,
\]
since $\zeta_\Nsite/(\sqrt d\,\sigma_\Nsite)
\le\bar{q}\,\sigma_1/\sigma_\Nsite
+\bar{q}\,(1+\mu_w)\sqrt{d\,e^{\Lambda}C_0\,\Delta_{\max}}/\sigma_1$. Both constants
are depth independent. The factor $d$ inside $C'$ comes from bounding the
operator norm of $\bm{\Phi}_{\Nsite,t+2}\mathbf{J}_{t+1}$ by its Frobenius norm,
and it multiplies only the $\log^2$ term.

The statement is for the fixed endpoint $\Nsite$. For a prefix endpoint $n$ the
same proof gives the constants with $\sigma_n$ in place of $\sigma_\Nsite$, and
since $\sigma_1/\sigma_n\le1$, all of them are dominated by the uniform constants
\[
  \bar C:=8\bigl[\tau_2+(1+\mu_w)^2C_0e^{\Lambda}\bigr],
  \qquad
  \bar C':=\frac89\Bigl(1+(1+\mu_w)\,\frac{\sqrt{d\,e^{\Lambda}C_0\,\Delta_{\max}}}{\sigma_1}\Bigr)^2 .
\]
Covering all prefixes at level $p/\Nsite$ replaces $\lambda_p$ by
$\log(2\Nsite/p)$ via a union bound.
\end{proof}

\begin{proposition}[second-order remainder]\label{prop:remainder}
Under the hypotheses of Theorem~\ref{thm:prop}, fix $p\in(0,1)$, let $\Theta_n$ be the bound
that the proof of Theorem~\ref{thm:prop}(ii) gives for the prefix ending at site
$n$, at level $p/\Nsite$, let $\tau$ be the first site with
$\norm{\mathbf{E}_n}>2\Theta_n$ ($\tau=\Nsite$ if there is none), and set
$\mathbf{r}^\tau_n=\mathbf{r}_n$ for $n<\tau$ and $\mathbf{0}$ otherwise. Then
\[
  \Bigl(\mathbb{E}\,\Bigl\|\sum_{n<\Nsite}\bm{\Phi}_{\Nsite,n+1}\mathbf{r}^\tau_n\Bigr\|^2\Bigr)^{1/2}
  \;\le\;C_2\,\bar{q}^2\sqrt\Nsite\,\sqrt d\,\sigma_\Nsite,
\]
with $C_2$ growing at most as $\log^2(\Nsite/p)$. The sum equals the second-order
part $\mathbf{E}_\Nsite-\mathbf{E}_\Nsite^{\mathrm{lin}}$ whenever no site exceeds the
threshold.
\end{proposition}

\begin{proof}[Proof of Proposition~\ref{prop:remainder}]
Set $\kappa_p^2:=\bar C\log(2\Nsite/p)+\bar C'\log^2(2\Nsite/p)$,
the uniform constant with which part (ii), applied at level $p/\Nsite$ and
extended to all prefixes, bounds every prefix deviation, so that
$\Theta_n=\kappa_p\,\bar{q}\,\sqrt d\,\sigma_n$, with $\tau$ and
$\mathbf{r}^\tau_n$ as in the statement. The
truncated sum $\sum_{n<\Nsite}\bm{\Phi}_{\Nsite,n+1}\mathbf{r}^\tau_n$ has the
deterministic endpoint $\Nsite$, and on the event $\tau=\Nsite$ it equals
$\mathbf{E}_\Nsite-\mathbf{E}^{\mathrm{lin}}_\Nsite$, the second-order part of
Eqn.~\eqref{eq:carry}.

For $n<\tau$, R3 and the definition of $\tau$ give
\[
  \norm{\mathbf{r}^\tau_n}
  \;\le\;\tfrac12H_n\norm{\mathbf{E}_n}^2
  \;\le\;\frac{\bar h}{2\sigma_n^2\sqrt d}\,(2\Theta_n)^2
  \;=\;2\bar h\,\kappa_p^2\,\bar{q}^2\sqrt d .
\]
The covariances of the $\mathbf{r}^\tau_n$ form the third collection of
Eqn.~\eqref{eq:family}, so Lemma~\ref{lem:covprop}, Minkowski's inequality, and
$\sum_n1/\sigma_{n+1}\le2\sigma_\Nsite/\Delta_{\min}$ (from R1) give
\begin{align*}
  \Bigl\|\sum_{n<\Nsite}\bm{\Phi}_{\Nsite,n+1}\mathbf{r}^\tau_n\Bigr\|_{L^2}
  \;\le\;\sum_n e^{\Lambda/2}\,\frac{\sigma_\Nsite}{\sigma_{n+1}}\,
        \bigl\|\mathbf{r}^\tau_n\bigr\|_{L^2}
  \;\le\;2\bar h\,\kappa_p^2\,e^{\Lambda/2}\bar{q}^2\sqrt d\;
        \frac{2\sigma_\Nsite^2}{\Delta_{\min}} .
\end{align*}
Since $\sigma_\Nsite\le\sigma_1+\sqrt{\Delta_{\max}\Nsite}
\le(\sigma_1+\sqrt{\Delta_{\max}})\sqrt\Nsite$,
\[
  \frac{4\bar h\,\kappa_p^2\,e^{\Lambda/2}}{\Delta_{\min}}\,
  \bar{q}^2\sqrt d\;\sigma_\Nsite^2
  \;\le\;
  \underbrace{\frac{4\bar h\,e^{\Lambda/2}(\sigma_1+\sqrt{\Delta_{\max}})}{\Delta_{\min}}\,
  \kappa_p^2}_{=:C_2}\;
  \bar{q}^2\sqrt\Nsite\,\sqrt d\,\sigma_\Nsite .
  \qedhere
\]
\end{proof}

\begin{corollary}[second-order budget]\label{cor:budget}
If in addition
\begin{equation}\label{eq:budget}
  \sum_{n}\frac{\bigl(\mathbb{E}\norm{\mathbf{r}^\tau_n}^2\bigr)^{1/2}}{\sqrt d\,\sigma_{n+1}}
  \;\le\;\Lambda_2\,\bar{q}^2
\end{equation}
with $\Lambda_2$ independent of $\Nsite$, then
\[
  \Bigl(\mathbb{E}\,\Bigl\|\sum_{n<\Nsite}\bm{\Phi}_{\Nsite,n+1}\mathbf{r}^\tau_n\Bigr\|^2\Bigr)^{1/2}
  \;\le\;\sum_n e^{\Lambda/2}\,\frac{\sigma_\Nsite}{\sigma_{n+1}}\,
        \bigl\|\mathbf{r}^\tau_n\bigr\|_{L^2}
  \;\le\;e^{\Lambda/2}\Lambda_2\,\bar{q}^2\sqrt d\,\sigma_\Nsite,
\]
with no factor $\sqrt\Nsite$.
\end{corollary}

Condition~\eqref{eq:budget} asks that the total second-order response along a pass
stay at the level of a single response. Standard post-training bias correction
shows that quantization-induced mean shifts can be compensated
explicitly~\citep{nagel2019}, which motivates the mean part of this budget without
proving it.

\subsection{Sharpness}

\begin{proposition}[sharpness]\label{prop:sharp}
(a) There is a network satisfying R1--R3 and the noise model in which
$\mathbb{E}\|\mathbf{E}^{\mathrm{lin}}_\Nsite\|^2/(d\sigma_\Nsite^2)$ grows
exponentially in $\sqrt\Nsite$. Its budget $\Lambda$ in
Assumption~\ref{asm:noreinf} grows as $\sqrt\Nsite$. (b)
There is a network satisfying R1--R3, the noise model, and
Eqn.~\eqref{eq:asm-trace} with every $\varepsilon_k=0$ for every positive
semidefinite $\bm{\Sigma}$, in which the stopped remainder of
Proposition~\ref{prop:remainder} attains the order $\bar{q}^2\sqrt\Nsite$.
\end{proposition}

\begin{proof}[Proof of (a)]
Fix orthonormal vectors $\mathbf{e},\mathbf{e}_0$, a constant $c>0$, and the
scale path $\sigma_k^2=\sigma_0^2+k\Delta$, and take the affine blocks
\[
  F_k(\mathbf{x})=\frac{c}{\sigma_k}\,\mathbf{e}\mathbf{e}^\top\mathbf{x}
  +\sqrt d\,(\sigma_{k+1}-\sigma_k)\,\mathbf{e}_0
\]
on the clean ray $\mathbf{z}_k=\sqrt d\,\sigma_k\mathbf{e}_0$. Then
$\mathbf{z}_{k+1}=\mathbf{z}_k+F_k(\mathbf{z}_k)$ stays on the ray,
\[
  \mathbf{J}_k=\frac{c}{\sigma_k}\,\mathbf{e}\mathbf{e}^\top,
  \qquad
  (\mathbf{I}+\mathbf{J}_k)\mathbf{e}=\Bigl(1+\frac{c}{\sigma_k}\Bigr)\mathbf{e},
\]
and $H_k=0$, so R3 is trivial and R1 and R2 hold with constants depending on
$c$, $\Delta$, and $d$. Let the injected errors be independent and uniform on the
sphere of radius $\qbound$, so the noise model holds with conditional covariance
exactly $(\qbound^2/d)\,\mathbf{I}$ and Lemma~\ref{lem:msid} is an equality.

In Lemma~\ref{lem:ksmat} all terms of $\mathbf{K}_s\mathbf{e}$ carry one sign, so
for an early site $s$,
\[
  \norm{\mathbf{K}_s\mathbf{e}}
  \;\ge\;T_0\,\frac{c}{\sigma_{s+1}}\prod_{\ell=s+2}^{\Nsite-1}\Bigl(1+\frac{c}{\sigma_\ell}\Bigr),
  \qquad
  \sum_{\ell}\log\Bigl(1+\frac{c}{\sigma_\ell}\Bigr)\;\asymp\;2c\sqrt{\Nsite/\Delta}.
\]
Hence
\[
  \mathbb{E}\,\|\mathbf{E}^{\mathrm{lin}}_\Nsite\|^2
  =\frac{\qbound^2}{d}\sum_{s'}\norm{\mathbf{K}_{s'}}_F^2
  \;\ge\;\frac{\qbound^2}{d}\,\norm{\mathbf{K}_{s}\mathbf{e}}^2,
\]
which grows exponentially in $\sqrt\Nsite$ even after dividing by
$d\sigma_\Nsite^2$.

It remains to check that the network violates
Assumption~\ref{asm:noreinf}. An injection at site $s$ produces at site $k$ the
deviation $\mathbf{v}=\bm{\Phi}_{k,s+1}\injnoise_s$ with covariance proportional
to
$\bm{\Phi}_{k,s+1}\bm{\Phi}_{k,s+1}^\top
=\mathbf{I}-\mathbf{e}\mathbf{e}^\top+\pi_k^2\,\mathbf{e}\mathbf{e}^\top$, where
$\pi_k=\prod_{\ell=s+1}^{k-1}(1+c/\sigma_\ell)\to\infty$. Since
$\mathbf{I}+\mathbf{J}_k$ fixes $\mathbf{e}^\perp$ and stretches $\mathbf{e}$ by
$1+c/\sigma_k$,
\[
  \frac{\mathbb{E}\|(\mathbf{I}+\mathbf{J}_k)\mathbf{v}\|^2}{\mathbb{E}\norm{\mathbf{v}}^2}
  =\frac{(d-1)+\pi_k^2(1+c/\sigma_k)^2}{(d-1)+\pi_k^2}
  \;\ge\;1+\frac{2c}{\sigma_k}\cdot\frac{\pi_k^2}{(d-1)+\pi_k^2}
  \;\ge\;1+\frac{c}{\sigma_k}
\]
once $\pi_k^2\ge d$. The assumption then requires
$(1+\varepsilon_k)(1+\Delta/\sigma_k^2)\ge1+c/\sigma_k$, so at every site with
$\pi_k^2\ge d$ and $\sigma_k\ge\max(2\Delta/c,\sqrt\Delta)$,
\[
  \varepsilon_k\;\ge\;\frac{c/\sigma_k-\Delta/\sigma_k^2}{1+\Delta/\sigma_k^2}
  \;\ge\;\frac{c}{4\sigma_k}.
\]
Both conditions hold from a fixed site on, with $\pi_k\to\infty$ at $s=0$ and
$\sigma_k\to\infty$, and the sum over the remaining sites,
$\sum_k c/(4\sigma_k)\asymp(c/2)\sqrt{\Nsite/\Delta}$, grows as $\sqrt\Nsite$, so no
$\Lambda$ independent of $\Nsite$ exists. Conversely,
$(\mathbf{I}+\mathbf{J}_k)^\top(\mathbf{I}+\mathbf{J}_k)\preceq(1+c/\sigma_k)^2\mathbf{I}$
gives $\varepsilon_k\le2c/\sigma_k+c^2/\sigma_k^2$, so the budget grows
exactly as $\sqrt\Nsite$.
\end{proof}

\begin{proof}[Proof of (b)]
Take orthonormal $\mathbf{e}_0,\mathbf{f}$,
$\bm{\Pi}=\mathbf{I}-\mathbf{e}_0\mathbf{e}_0^\top$,
$n(\mathbf{x})=\mathbf{x}/\norm{\mathbf{x}}$, a constant $h>0$, and the
$0$-homogeneous blocks
\[
  F_k(\mathbf{x})=b_k\,n(\mathbf{x})+\tfrac12h\sqrt d\,\norm{\bm{\Pi}n(\mathbf{x})}^2\,\mathbf{f},
  \qquad
  b_k=\sqrt d\,(\sigma_{k+1}-\sigma_k),
\]
on the clean ray $\mathbf{z}_k=\sqrt d\,\sigma_k\mathbf{e}_0$,
$\sigma_k^2=\sigma_0^2+k\Delta$. At the clean point the quadratic term has zero
first derivative, and
\[
  \mathbf{J}_k=\alpha_k\bm{\Pi},\quad 1+\alpha_k=\frac{\sigma_{k+1}}{\sigma_k},
  \qquad
  (\mathbf{I}+\mathbf{J}_k)^\top(\mathbf{I}+\mathbf{J}_k)
  =\mathbf{I}+\frac{\Delta}{\sigma_k^2}\,\bm{\Pi}
  \preceq\frac{\sigma_{k+1}^2}{\sigma_k^2}\,\mathbf{I},
\]
so Eqn.~\eqref{eq:asm-trace} holds with $\varepsilon_k=0$ for every positive
semidefinite $\bm{\Sigma}$, and R1--R3 hold.

Inject one error at site $0$, uniform on the sphere of radius $\qbound$, and none
later. Its tangential part is carried at the critical gain,
$\bm{\Phi}_{n,k}\big|_{\bm{\Pi}}=(\sigma_n/\sigma_k)\,\bm{\Pi}$, so
$\mathbf{X}_k:=\bm{\Pi}\mathbf{E}^{\mathrm{lin}}_k/\sigma_k$ satisfies
$\mathbf{X}_{k+1}=\mathbf{X}_k+\bm{\Pi}\bm{\rho}_k/\sigma_{k+1}$. For a single
injection Eqn.~\eqref{eq:kernel} gives $\bm{\rho}_0=\injnoise_0$ and
$\bm{\rho}_k=-w_{k-1}\injnoise_0$ for $k\ge1$, so
$\mathbf{X}_k=g_k\,\bm{\Pi}\injnoise_0$ with
\[
  g_k=\frac1{\sigma_1}-\sum_{t=1}^{k-1}\frac{w_{t-1}}{\sigma_{t+1}}
  =\frac{T_{k-1}}{\sigma_k}+\sum_{t=0}^{k-2}T_t\Bigl(\frac1{\sigma_{t+1}}-\frac1{\sigma_{t+2}}\Bigr),
\]
where the second form follows by summation by parts. Since
$g_{k+1}-g_k=-w_{k-1}/\sigma_{k+1}\le0$, the sequence decreases toward
$g_\infty=\sum_{t\ge0}T_t(1/\sigma_{t+1}-1/\sigma_{t+2})>0$, so $g_k\ge g_\infty$
for every $k$. Along $\mathbf{e}_0$ nothing is propagated, so only the tail of
the injected error remains,
$\langle\mathbf{e}_0,\mathbf{E}^{\mathrm{lin}}_k\rangle=T_{k-1}\langle\mathbf{e}_0,\injnoise_0\rangle$.

The $\mathbf{f}$-component of the second-order response at site $k$ collects two
curvatures. The quadratic term of the block gives
\[
  \frac{h}{2\sigma_k^2\sqrt d}\,\bigl\|\bm{\Pi}\mathbf{E}^{\mathrm{lin}}_k\bigr\|^2
  \bigl(1+O(\qbound)\bigr)
  \;=\;\frac{h}{2\sqrt d}\,g_k^2\,\norm{\bm{\Pi}\injnoise_0}^2\bigl(1+O(\qbound)\bigr),
\]
terms of one sign carrying the same random factor $\norm{\bm{\Pi}\injnoise_0}^2$,
of mean $\qbound^2(d-1)/d$. The curvature of $n(\mathbf{x})$ in the
$b_k\,n(\mathbf{x})$ term adds the cross term
\[
  -\frac{b_k}{d\sigma_k^2}\,
  \langle\mathbf{f},\mathbf{E}^{\mathrm{lin}}_k\rangle
  \langle\mathbf{e}_0,\mathbf{E}^{\mathrm{lin}}_k\rangle\bigl(1+O(\qbound)\bigr),
\]
of the same order in $\qbound$ but carrying the factor $T_{k-1}$, which decays
geometrically in $k$, and of mean zero. In the sums below the cross terms
therefore contribute $O(1)$ against $\asymp\sqrt\Nsite$ one-sign terms, and the
quadratic term dominates. The remaining second-order components of
$\mathbf{r}_k$ carry an extra factor $1/\sigma_k$ or $T_{k-1}$. Those along
$\mathbf{e}_0$ are orthogonal to $\mathbf{f}$, the $\mathbf{f}$-part of those along
$\bm{\Pi}\injnoise_0$ is the cross term above, and neither changes the sums. For each $\Nsite$ take $\qbound$ small enough that the
deviation stays within its linear bound, so the truncation is never active and
$\mathbf{r}^\tau_n=\mathbf{r}_n$. This is possible because
$\mathbf{E}^{\mathrm{lin}}_n$ and the threshold $\Theta_n$ both scale linearly in
$\qbound$ while the remainder is quadratic. The propagated responses then add
coherently along $\mathbf{f}$,
\[
  \sum_{n}\frac{(\mathbb{E}\norm{\mathbf{r}^\tau_n}^2)^{1/2}}{\sqrt d\,\sigma_{n+1}}
  \;\asymp\;\bar{q}^2\sqrt\Nsite,
  \qquad
  \frac1{d\sigma_\Nsite^2}\,
  \mathbb{E}\Bigl\|\sum_n\bm{\Phi}_{\Nsite,n+1}\mathbf{r}^\tau_n\Bigr\|^2
  \;\asymp\;\bar{q}^4\Nsite .
\]
The network satisfies every hypothesis of Theorem~\ref{thm:prop} with
$\varepsilon_k=0$, so no argument from those hypotheses alone can remove the
$\sqrt\Nsite$ of Proposition~\ref{prop:remainder}.
\end{proof}

\subsection{The backward source bound}

On the backward pass the compensator runs the same recurrence along the backward
site sequence (Sec.~\ref{subsec:backward-impl}), so at gain one the correction
applied at a site is again built from the errors injected at the sites visited
before it, with the weights of Eqn.~\eqref{eq:kernel}. We keep the forward index
throughout. The backward pass visits the sites in decreasing order of $k$, so the
error injected at site $s$ is corrected at the sites $k=s-1,s-2,\dots$, where
$\sigma_k\le\sigma_s$. What changes is the point of entry. An error
or a correction handled at site $k$ reaches the gradient of the hidden state
through the Jacobian of the normalization at that site
(Sec.~\ref{subsec:noise-process}). Write this operator at site $k$ as
$\mathbf{A}_k=\mathbf{D}_k/\sigma_k$ with $\norm{\mathbf{D}_k}\le\bar\gamma$,
where $\bar\gamma$ is a constant of the order of the largest magnitude
among the normalization weights. The
error $\injnoise_s$ then enters through $\mathbf{A}_s$ at its own site, while the
fraction $w_{s-1-k}\injnoise_s$ subtracted at a later visited site $k<s$ enters
through $\mathbf{A}_k$. Once the backward pass has reached site $t$, the
compensated partial sum is
\[
  \mathbf{S}^{\mathrm{bwd}}_t=\sum_{s\ge t}\mathbf{B}_{s,t}\,\injnoise_s,
  \qquad
  \mathbf{B}_{s,t}=\mathbf{A}_s-\sum_{k=t}^{s-1}w_{s-1-k}\,\mathbf{A}_k .
\]
A correction now leaves through a different operator than the one its error
entered through, so the cancellation is no longer exact, and the bound depends on
how fast the operators vary along the trajectory.

\begin{proposition}[backward source]\label{prop:bwdsrc}
Suppose, in addition to R1, that the backward injections satisfy the noise model
of Eqn.~\eqref{eq:noise-model} with the history taken in the order in which the
backward pass visits the sites and with $\qbound$ their bound in gradient units, and
that the operators vary smoothly along
the trajectory,
$\norm{\mathbf{D}_s-\mathbf{D}_k}\le c_D\,(s-k)\sqrt{\Delta_{\max}}/\sigma_s$
for $k\le s$. Then
\[
  \max_t\;\mathbb{E}\,\bigl\|\mathbf{S}^{\mathrm{bwd}}_t\bigr\|^2
  \;\le\;C_{\mathrm{bwd}}\;\frac{\qbound^2}{\sigma_1^2},
\]
with $C_{\mathrm{bwd}}$ independent of $\Nsite$, built from $\bar\gamma$, $c_D$,
the kernel moments $\mu_w$ and $\tau_2$, and the scale constants
$\sigma_0,\sigma_1,\Delta_{\min},\Delta_{\max}$.
\end{proposition}

\begin{proof}
Since $\sum_rw_r=1$,
\[
  \mathbf{B}_{s,t}
  =\mathbf{A}_s\,T_{s-t}
  +\sum_{k=t}^{s-1}w_{s-1-k}\,(\mathbf{A}_s-\mathbf{A}_k).
\]
For $k<s$, R1 gives $\sigma_s^2-\sigma_k^2\le(s-k)\Delta_{\max}$, hence
$\sigma_s-\sigma_k\le(s-k)\Delta_{\max}/\sigma_s$, and with $\sigma_k\ge\sigma_0$
and the smoothness hypothesis,
\[
  \norm{\mathbf{A}_s-\mathbf{A}_k}
  \;\le\;\frac{\norm{\mathbf{D}_s-\mathbf{D}_k}}{\sigma_s}
  +\bar\gamma\,\frac{\sigma_s-\sigma_k}{\sigma_s\sigma_k}
  \;\le\;\frac{C_{\mathrm{A}}\,(s-k)}{\sigma_s^2},
  \qquad
  C_{\mathrm{A}}=c_D\sqrt{\Delta_{\max}}+\frac{\bar\gamma\,\Delta_{\max}}{\sigma_0}.
\]
The kernel has $\sum_rw_r(r+1)=1+\mu_w$, the Frobenius norms carry a factor
$\sqrt d$, and Lemma~\ref{lem:msid} applies with $\mathbf{B}_{s,t}$ in place of
$\mathbf{K}_s$ and the history taken in visiting order, so
\begin{align*}
  \mathbb{E}\,\bigl\|\mathbf{S}^{\mathrm{bwd}}_t\bigr\|^2
  \;\le\;\frac{\qbound^2}{d}\sum_{s\ge t}\norm{\mathbf{B}_{s,t}}_F^2
  &\;\le\;2\qbound^2\,\bar\gamma^2\sum_{s\ge t}\frac{T_{s-t}^2}{\sigma_s^2}
  +2\qbound^2\,C_{\mathrm{A}}^2(1+\mu_w)^2\sum_{s\ge t}\frac{1}{\sigma_s^4}\\
  &\;\le\;\frac{2\qbound^2}{\sigma_1^2}\Bigl[\bar\gamma^2\,c_\Delta
  +C_{\mathrm{A}}^2(1+\mu_w)^2\,\frac{c_\Delta'}{\sigma_1^2}\Bigr],
\end{align*}
where $c_\Delta:=\sup_t\sum_{s\ge t}T_{s-t}^2\,\sigma_1^2/\sigma_s^2$ and
$c_\Delta':=\sum_s\sigma_1^4/\sigma_s^4$. Both are finite with explicit bounds:
$\sigma_s\ge\sigma_0$ gives $c_\Delta\le(\sigma_1^2/\sigma_0^2)\,\tau_2$, and
$\sigma_s^4\ge(\sigma_0^2+s\Delta_{\min})^2$ gives
$c_\Delta'\le\sigma_1^4/\sigma_0^4+\sigma_1^4/(\Delta_{\min}\sigma_0^2)$ by
comparison with the integral. Neither depends on $\Nsite$.
\end{proof}

\looseness=-1 The backward pass therefore has the analogous source guarantee, in mean
square and under a smoothness condition on the normalization Jacobians. The guarantee
follows from the same cancellation mechanism, with constants degraded by the operator mismatch
between a site and the sites that correct it. 
\section{Training Configuration}\label{app:train}

The training runs of Sec.~\ref{subsec:exp-train} use the NeMo Qwen3-0.6B
configuration, trained from scratch on Dolma C4: 28 layers, hidden size 1,024, FFN size
3,072, 16 attention heads with 8 query groups, RMSNorm, QK layer normalization,
RoPE base $10^6$, and no dropout. The 151,669-token vocabulary of the Qwen
tokenizer is padded to 151,936. Sequences have 2,048 tokens. The main run uses a
micro batch of 24 and a global batch of 384 sequences, 786,432 tokens per step,
for 12,716 steps, about 10.0B tokens in total. Parallelism is $\TP=4$ without
pipeline, context, data, or sequence parallelism. Training uses BF16 mixed
precision with FP32 master weights, gradients, and optimizer states. The
optimizer is AdamW with $\beta_1=0.9$, $\beta_2=0.95$, $\epsilon=10^{-5}$,
weight decay $0.1$, and gradient clipping at $1.0$. The learning rate
warms up over 100 steps to its peak, $3.0\times10^{-4}$ for the 16-bit run and
$3.5\times10^{-4}$ for the 4-bit run, and follows a cosine decay to
$3\times10^{-5}$. The learning-rate ablation of Table~\ref{tab:train}b uses the
same configuration with a global batch of 32 sequences for 4,096 steps, about
0.27B tokens. The reported loss is the training cross-entropy averaged over
tokens.

\section{Implementation Details}\label{app:abi}

The pipeline of Sec.~\ref{subsec:pipeline} keeps four stages busy at once on
each of the $\TP$ GPUs, and a tile has to move from one stage to the next
without the host taking part. Three things have to be in place beforehand. Each
GPU has to know where to write inside its peers. The stages within a GPU have to
overlap across tiles while the stages of any one tile stay in order. Finally, a GPU
has to detect when the data it needs from a peer has arrived.
The rest of this appendix describes these in turn, then the cache that keeps
their setup out of the per-call cost, and finally the exactness of the
backward rescaling of Sec.~\ref{subsec:backward-impl}.

\begin{figure}[!htb]
\centering
\begin{tikzpicture}[ecfig]
\def\cw{0.95}\def\segp{2.00}\def\rh{0.48}
\def\clr#1{\ifcase#1 pA\or pB\or pC\or pD\fi}

\def\rowA{{"0.8","-1.4","2.1","0.6","-0.5","1.9","1.2","-0.9"}}
\def\rowB{{"-1.1","0.7","0.4","-2.2","1.6","0.3","-0.8","1.5"}}
\def\rowC{{"1.3","0.9","-1.8","1.1","0.7","-1.4","2.0","0.4"}}
\def\rowD{{"0.6","-0.8","1.5","0.2","-1.2","0.9","0.3","-1.7"}}
\foreach \rr in {0,1,2,3}{
  \pgfmathtruncatemacro{\gn}{\rr}
  \node[hdr,anchor=east] at (-0.62,{-\rr*\rh}) {GPU~\gn};
  \foreach \kk in {0,...,7}{
    \pgfmathtruncatemacro{\jj}{\kk/2}
    \pgfmathsetmacro{\X}{\jj*\segp+mod(\kk,2)*\cw}
    \ifcase\rr \pgfmathsetmacro{\val}{\rowA[\kk]}\or \pgfmathsetmacro{\val}{\rowB[\kk]}%
      \or \pgfmathsetmacro{\val}{\rowC[\kk]}\or \pgfmathsetmacro{\val}{\rowD[\kk]}\fi
    \ifnum\rr=\jj
      \node[draw=\clr{\jj}!60!black,line width=1.1pt,fill=\clr{\jj}!42!white,
            minimum width=\cw cm,minimum height=\rh cm,inner sep=0pt,font=\fignote]
        at (\X,{-\rr*\rh}) {$\pgfmathprintnumber[fixed,fixed zerofill,precision=1]{\val}$};
    \else
      \node[draw=\clr{\jj}!85,line width=0.4pt,fill=\clr{\jj}!16!white,
            minimum width=\cw cm,minimum height=\rh cm,inner sep=0pt,font=\fignote]
        at (\X,{-\rr*\rh}) {$\pgfmathprintnumber[fixed,fixed zerofill,precision=1]{\val}$};\fi}}
\foreach \jj in {0,1,2,3}{\node[hdr,anchor=south] at ({\jj*\segp+0.475},0.36) {segment \jj};}
\node[note,anchor=north] at (3.71,-1.78) {each GPU quantizes the three segments it does not own};

\draw[flow] (7.70,-0.72)
  -- node[midway,above=3pt,font=\figlab,fill=white,inner sep=1pt] {all-to-all} (8.80,-0.72);

\begin{scope}[xshift=9.20cm]
\def\sumA{{"1.57","-0.64"}}\def\sumB{{"2.20","-0.36"}}
\def\sumC{{"0.56","1.59"}}\def\sumD{{"2.64","-0.77"}}
\foreach \rr in {0,1,2,3}{
  \foreach \kk in {0,1,2,3}{
    \ifnum\rr=\kk
      \node[draw=\clr{\rr}!60!black,line width=1.1pt,fill=\clr{\rr}!42!white,
            minimum width=0.28cm,minimum height=\rh cm,inner sep=0pt] at ({\kk*0.32},{-\rr*\rh}) {};
    \else
      \node[draw=\clr{\rr}!85,line width=0.4pt,densely dashed,fill=\clr{\rr}!10!white,
            minimum width=0.28cm,minimum height=\rh cm,inner sep=0pt] at ({\kk*0.32},{-\rr*\rh}) {};\fi}
  \node[font=\figlab] at (1.46,{-\rr*\rh}) {$\Sigma$};
  \foreach \kk in {0,1}{
    \ifcase\rr \pgfmathsetmacro{\val}{\sumA[\kk]}\or \pgfmathsetmacro{\val}{\sumB[\kk]}%
      \or \pgfmathsetmacro{\val}{\sumC[\kk]}\or \pgfmathsetmacro{\val}{\sumD[\kk]}\fi
    \node[draw=\clr{\rr}!85,line width=0.4pt,fill=\clr{\rr}!16!white,
          minimum width=\cw cm,minimum height=\rh cm,inner sep=0pt,font=\fignote]
      at ({2.16+\kk*\cw},{-\rr*\rh}) {$\pgfmathprintnumber[fixed,fixed zerofill,precision=2]{\val}$};}}
\node[hdr,anchor=south,align=center] at (1.80,0.36) {the owner adds four copies\\ in high precision};
\node[note,anchor=north,align=center] at (1.80,-1.86) {the solid copy is the owner's own,\\ the dashed ones crossed a link};
\end{scope}

\draw[flow,rounded corners=3pt]
     (13.25,-1.78) -- (13.25,-2.72)
  -- node[midway,below=3pt,font=\figlab,fill=white,inner sep=1pt] {all-gather} (7.80,-2.72);

\def\outv{{"1.57","-0.67","2.20","-0.31","0.45","1.59","2.64","-0.76"}}
\def\exv{{"1.60","-0.60","2.20","-0.30","0.60","1.70","2.70","-0.70"}}
\def\dev{{"-0.03","-0.07","0.00","-0.01","-0.15","-0.11","-0.06","-0.06"}}
\foreach \kk in {0,...,7}{
  \pgfmathtruncatemacro{\jj}{\kk/2}
  \pgfmathsetmacro{\X}{\jj*\segp+mod(\kk,2)*\cw}
  \pgfmathsetmacro{\val}{\outv[\kk]}
  \node[draw=\clr{\jj}!85,line width=0.4pt,fill=\clr{\jj}!16!white,
        minimum width=\cw cm,minimum height=\rh cm,inner sep=0pt,font=\fignote]
    at (\X,-2.72) {$\pgfmathprintnumber[fixed,fixed zerofill,precision=2]{\val}$};
  \pgfmathsetmacro{\wal}{\exv[\kk]}
  \node[draw=black!30,line width=0.4pt,fill=white,
        minimum width=\cw cm,minimum height=\rh cm,inner sep=0pt,font=\fignote]
    at (\X,-3.20) {$\pgfmathprintnumber[fixed,fixed zerofill,precision=2]{\wal}$};
  \pgfmathsetmacro{\dvv}{\dev[\kk]}
  \node[draw=carnelian!40,line width=0.4pt,fill=carnelian!5,text=carnelian,
        minimum width=\cw cm,minimum height=\rh cm,inner sep=0pt,font=\fignote]
    at (\X,-3.68) {$\pgfmathprintnumber[fixed,fixed zerofill,precision=2]{\dvv}$};}
\node[hdr,anchor=east] at (-0.62,-2.72) {Echo output};
\node[hdr,anchor=east] at (-0.62,-3.20) {exact sum};
\node[hdr,anchor=east,text=carnelian] at (-0.62,-3.68) {error $\bm{\delta}$};
\node[note,anchor=north] at (3.71,-4.02)
  {the same codes reach every GPU, so this row is what every GPU holds};
\end{tikzpicture}
\caption{\textbf{Worked example at $\TP=4$.} The quantizer is
Eqn.~\eqref{eq:qsym} with $\bits=4$ and a group equal to the two values of one
segment. Top left: each GPU quantizes the three segments it does not own and
keeps its own segment in full precision. Top right: the owner adds the four
copies in high precision, three of them dequantized from $\bits$-bit codes and
one its own. Bottom: the all-gather sends one quantized copy of each segment
sum, so every GPU holds the same values. The last row is the error this
all-reduce injects, the $\injnoise_\site$ of
Fig.~\ref{fig:noise-process}. Inputs are shown to one decimal and sums to two.}
\label{fig:decomp-detailed}
\end{figure}
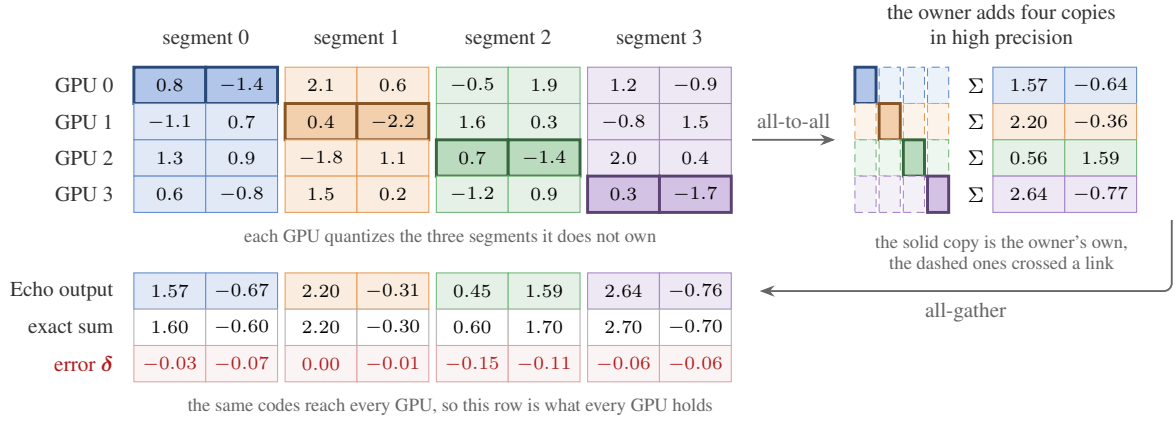

\noindent\textbf{Workspace layout.}
Before a GPU can write into a peer, it must hold that peer's addresses. A call
uses three kinds of memory on every GPU: a buffer that receives the $\bits$-bit
codes its peers send, a buffer that receives the scale and zero-point that
decode those codes (Sec.~\ref{subsec:quant}), and an array of signal words
that peers write to announce that the data has arrived. Each GPU allocates its
own three and shares their addresses with the others when the tensor-parallel
group is created, which gives one table of $3\TP{+}1$ addresses: the $\TP$ code
buffers, the $\TP$ parameter buffers, the $\TP$ signal arrays, and one flag
base. The table is built once and stays fixed, and every later call reads the
same entries, so a kernel reaches any peer by indexing the table with its rank.
The compensator state $\*m^{(1)},\*m^{(2)}$ of
Sec.~\ref{subsec:fusion} is kept out of this table and owned
by the compensation path, so the interface of the collective is the same with
and without compensation.

\noindent\textbf{Stream and event schedule.}
Push, reduce, gather, and tail each get their own CUDA stream. Work issued to
different streams runs concurrently, so putting the stages on separate streams
lets push quantize tile $3$ while gather is still sending tile $1$.
Streams do not enforce the order within a tile, and tile $2$ must not be
reduced before it has been pushed. CUDA events enforce this order. When
push finishes tile $2$ it records an event, and reduce waits on that event
before it starts tile $2$. Since a stage waits only for the stage directly
before it, the event of a tile is no longer needed once the next stage has
started on the tile, and two event slots used in turn suffice for the whole pipeline
(Fig.~\ref{fig:decomposition}(b)). The caller is synchronized at the two ends. It
records an event before the first push, so the pipeline starts only once the
input is ready, and waits on one after the last tail, so its next kernel starts
only once the output is complete. The compensator update is issued on a fifth
stream after the last tail and runs alongside the caller's next matrix
multiplication, and the entry stage
of the next call waits on an event recorded at the end of the update. The host launches the kernels and returns. It
performs no synchronization while the tiles are in flight.

\noindent\textbf{Flag values.}
Events do not cross GPUs. A stage that needs a segment from a peer therefore
polls the flag of Sec.~\ref{subsec:pipeline} and waits until it reads the
value that stands for the data it needs. Each stage publishes a value one higher
than the stage before it. Push publishes the base value, reduce publishes the
base plus one, and gather publishes the base plus two, which is the value tail
polls for. The host raises the base by three at every call, so no value is used
twice in the life of the group, and no flag needs to be cleared.
A consumer that arrives early reads a smaller number and keeps polling, and a
flag still holding a value from the previous call is below anything the current
call polls for, so it cannot release a consumer by mistake. No slot is cleared
between calls and no CPU round trip takes place on the readiness path.

\noindent\textbf{Schedule cache.}
The tiling, the launch order, and the streams and events above depend only on
the shape of the call, so they are computed once and reused, with a cache keyed
by the data type, bit width, policy, group size, and segment size. The cache
stores the tiling and launch table together with the streams and events for each
device. Only the first call for a new shape, or the first after a rebuild, pays
this build and warm-up cost, which the latency measurements of
Sec.~\ref{sec:experiments} exclude, so all reported numbers are steady state.

\noindent\textbf{Backward rescaling.}
Section~\ref{subsec:backward-impl} multiplies the gradient by a shared
power-of-two factor before the quantized all-reduce and divides the result by
the same factor after it. Three choices make the scaling exact. The factor is
a power of two, so multiplying and dividing shift only the floating-point
exponent and add no rounding of their own. Its exponent is derived from the
largest gradient magnitude across the tensor-parallel group, obtained with a
$\max$ all-reduce of a single scalar, so every GPU applies the same factor
and the all-reduce computes the same sum, only rescaled. Finally, the two steps of
the compensator run outside the scaling. The compensator forms $\*x'=\*x+\*c$ before the
multiplication and $\*u=\*x'-\hat{\*y}/\TP$ after the division, so
$\*m^{(1)}$, $\*m^{(2)}$ stay in original gradient units and the analysis of
Secs.~\ref{subsec:compensator} and~\ref{sec:theory} applies unchanged.

\end{document}